\documentclass[letterpaper,11pt]{article}
\usepackage[utf8]{inputenc}
\usepackage[T1]{fontenc}
\usepackage[mathscr]{eucal}

\usepackage[compact]{titlesec}
\usepackage{mathpazo,tgcursor,tgtermes}
\usepackage{textcomp}
\usepackage[margin=1in]{geometry}
\usepackage{bbm}
\usepackage{booktabs}
\usepackage{amsmath}
\usepackage{authblk}
\usepackage{amssymb}
\usepackage{graphicx}
\usepackage{mathtools}
\usepackage{tabulary}

\usepackage{xparse}
\usepackage{comment}
\usepackage{color}
\usepackage{amsmath}
\usepackage{amsthm}
\usepackage[linesnumbered, boxed]{algorithm2e}
\usepackage{tikz}
\usepackage{varwidth}
\usepackage[pagebackref,bookmarksnumbered,bookmarksopen,plainpages=false,pdfpagelabels,%
unicode,
breaklinks,colorlinks,citecolor=blue,linkcolor=blue,hyperindex]{hyperref}
\makeatletter 
\let\orgdescriptionlabel\descriptionlabel
\renewcommand*{\descriptionlabel}[1]{%
  \let\orglabel\label
  \let\label\@gobble
  \phantomsection
  \protected@edef\@currentlabel{#1}%
  \let\label\orglabel
  \orgdescriptionlabel{#1}%
}
\makeatother
\DeclareMathOperator{\fc}{fc_{LP}}

\DeclareMathOperator{\fcSDP}{fc_{SDP}}
\DeclareMathOperator{\rank}{rk}
\newcommand{\nnegrk}{\rank_{+}} 
\newcommand{\psdrk}{\rank_{\textnormal{psd}}} 

\DeclareMathOperator{\OPT}{OPT}

\DeclareMathOperator{\lca}{lca}
\DeclareMathOperator{\val}{val}
\DeclareMathOperator{\cost}{cost}
\DeclareMathOperator{\leaves}{leaves}
\DeclareMathOperator{\dist}{dist}
\newtheorem{theorem}{Theorem}[section]
\newtheorem{lemma}[theorem]{Lemma}
\newtheorem{remark}[theorem]{Remark}

\newtheorem{definition}[theorem]{Definition}
\newtheorem{corollary}[theorem]{Corollary}
\newtheorem{observation}[theorem]{Observation}
\newcommand{\R}{\mathbb{R}}
\newcommand*{\size}[1]{\left|#1\right|}

\newcommand{\sprod}[2]{\langle #1, #2 \rangle}

\newcommand{\norm}[1]{\left\Vert #1 \right\Vert}
\NewDocumentCommand{\ball}{mmmg}
{\ensuremath{%
    \IfNoValueTF{#4}
    {\mathcal{B}\left(#1, #2, #3\right)}%
    {\mathcal{B}_{#4}\left(#1, #2, #3\right)}
\xspace}}
\newcommand{\vol}[1]{\operatorname{vol}\left(#1\right)}

\newcommand{\diam}[1]{\operatorname{diam}\left(#1\right)}
\newcommand{\restr}[2]{\left.#1\right|_{#2}}

\newcommand*{\set}[2]{\left\{#1\,\middle|\,#2\right\}}
\newcommand{\err}[1]{\operatorname{err}\left(#1\right)}
\newcommand{\symM}{\ensuremath{\mathbb{S}}}
\newcommand{\Problem}[1]{\textnormal{\textsf{#1}}}

\NewDocumentCommand{\conv}{mo}{\operatorname{conv}\left(#1%
    \IfValueT{#2}{\,\middle|\,#2}%
  \right)}
\title{Hierarchical Clustering via Spreading Metrics}
\author[1]{Aurko Roy}
\affil[1]{College of Computing, Georgia Institute of Technology,
  Atlanta, GA,
  USA.
  \textit{Email:}~\texttt{aurko@gatech.edu}}

\author[2]{Sebastian Pokutta}
\affil[2]{ISyE, Georgia Institute of Technology,
  Atlanta, GA,
  USA.
  \textit{Email:}~\texttt{sebastian.pokutta@isye.gatech.edu}}
  
\begin{document}

\maketitle\begin{abstract}
We study the cost function for 
hierarchical clusterings introduced by \cite{DBLP:conf/stoc/Dasgupta16} 
where hierarchies are treated as first-class objects
rather than deriving their cost from projections into flat clusters. It was also
shown in \cite{DBLP:conf/stoc/Dasgupta16} that a top-down algorithm 
returns a hierarchical clustering of cost at most
\(O\left(\alpha_n \log n\right)\) times the cost of the optimal hierarchical clustering,
where \(\alpha_n\) is the approximation ratio of the Sparsest Cut subroutine used.
Thus using the best known approximation algorithm for Sparsest Cut due to Arora-Rao-Vazirani, 
the top-down algorithm returns a hierarchical clustering of cost at most 
\(O\left(\log^{3/2} n\right)\) times the cost of the optimal solution.
We improve this by giving an \(O(\log{n})\)-approximation algorithm for this problem.
Our main technical ingredients are a combinatorial characterization of ultrametrics induced
by this cost function, deriving an Integer Linear Programming (ILP) formulation for this family
of ultrametrics, and showing how to iteratively round an LP relaxation of this formulation by 
using the idea of \emph{sphere growing} which has been extensively used in the context of graph 
partitioning. We also prove that our algorithm returns an \(O(\log{n})\)-approximate 
hierarchical clustering for a generalization of this cost function also studied in \cite{DBLP:conf/stoc/Dasgupta16}.
Experiments show that the hierarchies found by using the ILP formulation as well 
as our rounding algorithm often have better projections into flat clusters than the standard
linkage based algorithms. We conclude with constant factor inapproximability results for this problem:
1) no polynomial size LP or SDP can achieve a constant factor approximation for this problem 
and 2) no polynomial time algorithm can achieve a constant factor approximation under the assumption
of the \emph{Small Set Expansion} hypothesis. 
\end{abstract}

\section{Introduction}
\emph{Hierarchical clustering} is an important method in cluster analysis where a data set is 
recursively partitioned into clusters of successively smaller size. They are typically
represented by rooted trees where the root corresponds to the entire data set, the leaves
correspond to individual data points and the intermediate nodes correspond to a cluster of its
descendant leaves. Such a hierarchy represents several possible \emph{flat clusterings} of the data 
at various levels of granularity; indeed every pruning of this tree returns a possible clustering. 
Therefore in situations where the number of desired clusters is not known beforehand, 
a hierarchical clustering scheme is often preferred to flat clustering.

The most popular algorithms for hierarchical clustering are bottoms-up agglomerative algorithms
like \emph{single linkage}, \emph{average linkage} and \emph{complete linkage}. In terms of theoretical
guarantees these algorithms are known to correctly recover a ground truth clustering if the 
similarity function on the data satisfies corresponding stability properties (see, e.g., \cite{balcan2008discriminative}).
Often, however, one wishes to think of a good clustering as optimizing some kind of cost function 
rather than recovering a hidden ``ground truth''.
This is the standard approach in the classical clustering setting where popular objectives are 
\(k\)-means, \(k\)-median, min-sum and \(k\)-center (see Chapter 14, \cite{friedman2001elements}). 
However as pointed out by \cite{DBLP:conf/stoc/Dasgupta16}
for a lot of popular hierarchical clustering algorithms including linkage based algorithms,
it is hard to pinpoint explicitly the cost function that these algorithms are optimizing. Moreover, much of the
existing cost function based approaches towards hierarchical clustering evaluate a hierarchy based on
a cost function for flat clustering, e.g., assigning the \(k\)-means or \(k\)-median cost to a pruning of this
tree. Motivated by this, \cite{DBLP:conf/stoc/Dasgupta16} introduced a cost function 
for hierarchical clustering where the cost takes into account the entire structure of the tree
rather than just the projections into flat clusterings. This cost function is shown to recover
the intuitively correct hierarchies on several synthetic examples like planted partitions and cliques.
In addition, a top-down graph partitioning
algorithm is presented that outputs a tree with cost at most
\(O(\alpha_n \log{n})\) times the cost of the optimal tree and
where \(\alpha_n\) is the approximation guarantee of the Sparsest Cut subroutine used. 
Thus using the Leighton-Rao algorithm~\cite{leighton1988approximate, leighton1999multicommodity}
or the Arora-Rao-Vazirani algorithm~\cite{arora2009expander} gives an approximation factor of 
\(O\left(\log^2{n}\right)\) and \(O\left(\log^{3/2}n\right)\) respectively.

In this work we give a polynomial time algorithm to recover a hierarchical clustering of cost
at most \(O(\log{n})\) times the cost of the optimal clustering according to this cost function.
We also analyze a generalization of this cost function studied by \cite{DBLP:conf/stoc/Dasgupta16}
and show that our algorithm still gives an \(O(\log{n})\) approximation in this setting.
We do this by viewing the cost function in terms of the ultrametric it induces on the data,
writing a convex relaxation for it and concluding by analyzing a popular rounding scheme
used in graph partitioning algorithms. We also implement the integer program, its LP relaxation,
and the rounding algorithm 
and test it on some synthetic and real world data sets to compare the cost of the rounded solutions
to the true optimum as well as to compare its performance to other hierarchical clustering
algorithms used in practice. Our experiments suggest that the hierarchies found by this algorithm are often
better than the ones found by linkage based algorithms as well as the \(k\)-means algorithm in terms of the error of
the best pruning of the tree compared to the ground truth.

\subsection{Related Work}
The immediate precursor to this work is \cite{DBLP:conf/stoc/Dasgupta16} where the cost function 
for evaluating a hierarchical clustering was introduced. Prior to this there has been a long line 
of research on hierarchical clustering in the context of phylogenetics and taxonomy (see, e.g.,
\cite{jardine1971mathematical, sneath1973numerical, felsenstein2004inferring}). Several authors have also
given theoretical justifications for the success of the popular linkage based algorithms 
for hierarchical clustering (see, e.g. \cite{jardine1968construction, zadeh2009uniqueness, ackerman2010characterization}).
In terms of cost functions, one approach has been to evaluate a hierarchy in terms of the \(k\)-means or \(k\)-median
cost that it induces (see \cite{dasgupta2005performance}). The cost function and the top-down algorithm in 
\cite{DBLP:conf/stoc/Dasgupta16} can also be seen as a theoretical justification for several graph 
partitioning heuristics that are used in practice. 

Besides this prior work on hierarchical clustering we are also motivated by the long line of work in the
classical clustering setting where a popular strategy is to study convex relaxations of these problems 
and to round an optimal fractional solution into an integral one with the aim of getting a good approximation to the cost function.
A long line of work (see, e.g., \cite{charikar1999constant, jain2001approximation,
jain2003greedy, charikar2012dependent}) has employed this approach on LP relaxations for the \(k\)-median problem,
including \cite{li2013approximating} which gives the best known approximation factor of \(1 + \sqrt{3} + \varepsilon\).
Similarly, a few authors have studied LP and SDP relaxations for the \(k\)-means problem
(see, e.g., \cite{peng2005new, peng2007approximating, awasthi2015relax}), while one of the best known
algorithms for kernel \(k\)-means and spectral clustering is due to \cite{recht2012factoring} which approximates
the nonnegative matrix factorization (NMF) problem by LPs.

LP relaxations for hierarchical clustering have also been studied in \cite{ailon2005fitting} where
the objective is to fit a tree metric to a data set given pairwise  dissimilarities. While the LP relaxation 
and rounding algorithm in \cite{ailon2005fitting} is similar in flavor, 
the result is incomparable to ours (see Section~\ref{sec:discussion} for a discussion). 
Another work that is indirectly related to our approach is \cite{di2015finding} where the authors
study an ILP to obtain a closest ultrametric to arbitrary functions on a discrete set.
Our approach is to give a combinatorial characterization of the ultrametrics induced by the cost function
of \cite{DBLP:conf/stoc/Dasgupta16} which allows us to use the tools from \cite{di2015finding} to model
the problem as an ILP. The natural LP relaxation of this ILP turns out to be closely related to LP relaxations 
considered before for several graph partitioning problems 
(see, e.g., \cite{leighton1988approximate, leighton1999multicommodity, even1999fast, krauthgamer2009partitioning})
and we use a rounding technique studied in this context to round this LP relaxation.

Recently, we became aware of independent work by \cite{charikar2016approximate} obtaining 
similar results for hierarchical clustering.
In particular \cite{charikar2016approximate} improve the approximation factor
to \(O\left(\sqrt{\log{n}}\right)\) by showing how to round a spreading metric SDP relaxation 
for this cost function. The analysis of this rounding procedure also enabled them
to show that the top-down heuristic of \cite{DBLP:conf/stoc/Dasgupta16}
actually returns an \(O(\sqrt{\log{n}})\) approximate clustering rather than an
\(O\left(\log^{3/2}{n}\right)\) approximate clustering. 
They also analyzed a very similar LP relaxation using the 
\emph{divide-and-conquer approximation algorithms using spreading metrics} paradigm
of \cite{even2000divide} together with a result of \cite{bartal2004graph} to show 
an \(O(\log{n})\) approximation. Finally, they also gave similar constant factor inapproximability
results for this problem.

\subsection{Contribution}
While studying convex relaxations of optimization problems is fairly natural,
for the cost function introduced in \cite{DBLP:conf/stoc/Dasgupta16} however, it is not immediately clear
how one would go about writing such a relaxation.
Our first contribution is to give a combinatorial characterization of the family of ultrametrics
induced by this cost function on hierarchies. Inspired by the approach in \cite{di2015finding} where
the authors study an integer linear program for finding the closest ultrametric, we are able to 
formulate the problem of finding the minimum cost hierarchical clustering as an integer linear
program. Interestingly and perhaps unsurprisingly, the specific family of ultrametrics 
induced by this cost function give rise to linear constraints studied before in the context of
finding balanced separators in weighted graphs. We then show how to round an optimal fractional solution using the 
\emph{sphere growing} technique first introduced in \cite{leighton1988approximate} (see also
\cite{garg1996approximate, even1999fast, charikar2003clustering}) to recover a tree of cost
at most \(O(\log{n})\) times the optimal tree for this cost function. 
The generalization of this cost function involves scaling every pairwise distances by an arbitrary
strictly increasing function \(f\) satisfying \(f(0) = 0\). We modify the integer linear program for
this general case and show that the rounding algorithm still finds a hierarchical clustering of
cost at most \(O(\log{n})\) times the optimal clustering in this setting.
We also show a constant factor inapproximability result for this problem for any polynomial sized
LP and SDP relaxations and under the assumption of the \emph{Small Set Expansion} hypothesis.
We conclude with an experimental study of the integer linear program and the rounding algorithm on
some synthetic and real world data sets to show that the approximation algorithm often recovers clusters close
to the true optimum (according to this cost function) and that its projections into flat clusters often
has a better error rate than the linkage based algorithms and the \(k\)-means algorithm.

\section{Preliminaries}\label{sec:preliminaries}
A similarity based clustering problem consists of a dataset \(V\) of \(n\) points and a 
\emph{similarity function} \(\kappa : V \times V \to \R_{\ge 0}\) such that 
\(\kappa(i, j)\) is a measure of the similarity between \(i\) and \(j\) for any \(i, j \in V\).
We will assume that the similarity function is symmetric i.e.,
\(\kappa(i, j) = \kappa(j, i)\) for every \(i, j \in V\). Note that we do not make any assumptions about
the points in \(V\) coming from an underlying metric space. For a given instance of a clustering problem
we have an associated weighted complete graph \(K_n\) with vertex set \(V\) and 
weight function given by \(\kappa\). A \emph{hierarchical clustering} of \(V\) is a tree \(T\)
with a designated root \(r\) and with the elements of \(V\) as its leaves, i.e., \(\leaves(T) = V\). 
For any set \(S \subseteq V\)
we denote the \emph{lowest common ancestor} of \(S\) in 
\(T\) by \(\lca(S)\). For pairs of points \(i, j \in V\) we will abuse the notation
for the sake of simplicity and denote \(\lca(\{i, j\})\)
simply by \(\lca(i, j)\). For a node \(v\) of \(T\) we denote the subtree of \(T\) rooted
at \(v\) by \(T[v]\). The following cost function was introduced by \cite{DBLP:conf/stoc/Dasgupta16}
to measure the quality of the hierarchical clustering \(T\)
\begin{align}
 \cost(T) \coloneqq \sum_{\{i, j\} \in E(K_n)} \kappa(i, j)\size{\leaves(T[\lca(i, j)])}\label{cost}.
\end{align}
The intuition behind this cost function is as follows. Let \(T\) be a hierarchical clustering with 
designated root \(r\) so that \(r\) represents the whole data set \(V\). Since \(\leaves(T) = V\),
every internal node \(v \in T\) represents a cluster of its descendant leaves, with the leaves themselves
representing singleton clusters of \(V\). Starting from \(r\) and going down the tree, 
every distinct pair of points \(i, j \in V\) will be eventually separated at the leaves.
If \(\kappa(i, j)\) is large, i.e., \(i\) and \(j\) are very similar to each other then we would like
them to be separated as far down the tree as possible if \(T\) is a good clustering of \(V\). 
This is enforced in the cost function~\eqref{cost}: if \(\kappa(i, j)\) is large then
the number of leaves of \(\lca(i, j)\) should be small i.e., \(\lca(i, j)\) should be far from the
root \(r\) of \(T\). Such a cost function is not unique however; 
see Section~\ref{sec:discussion} for some other cost functions of a similar flavor.

Note that while requiring \(\kappa\) to be non-negative might seem like an artificial restriction,
cost function~\eqref{cost} breaks down when all the \(\kappa(i, j) < 0\), since in this
case the trivial clustering \(r, T^*\) where \(T^*\) is the star graph with
\(V\) as its leaves is always the minimizer. 
Therefore in the rest of this work we will assume that \(\kappa \ge 0\). 
This is not a restriction compared to \cite{DBLP:conf/stoc/Dasgupta16}, since the 
Sparsest Cut algorithm used as a subroutine also
requires this assumption.
Let us now briefly recall the notion of an ultrametric. 

\begin{definition}[Ultrametric]\label{def:ultrametric}
 An \emph{ultrametric} on a set \(X\) of points is a distance function \(d: X \times X \to
 \R\) satisfying the following properties for every \(x, y, z \in X\)
 \begin{enumerate}
  \item \textbf{Nonnegativity:} \(d(x, y) \ge 0\) with \(d(x, y) = 0\) iff \(x = y\)
  
  \item \textbf{Symmetry:} \(d(x, y) = d(y, x)\)
  
  \item \textbf{Strong triangle inequality:} \(d(x, y) \le \max\{d(y, z), d(z, x)\}\)
 \end{enumerate}

\end{definition}

Under the cost function~\eqref{cost}, one can interpret the tree \(T\) 
as inducing an ultrametric \(d_T\) on \(V\) given by 
\(d_T(i, j) \coloneqq \size{\leaves(T[\lca\left(i, j\right)])} - 1\). This is an ultrametric since 
\(d_T(i, j) = 0\) iff \(i = j\) and for any triple \(i, j, k \in V\) we have \(d_T(i, j) \le 
\max\{d_T(i, k), d_T(j, k)\}\). 
The following definition introduces the notion of \emph{non-trivial ultrametrics}. These turn 
out to be precisely the ultrametrics that are induced by tree decompositions of \(V\) corresponding
to cost function~\eqref{cost}, as we will show in Corollary~\ref{cor:equivalent}.

\begin{definition}\label{def:non-trivial}
An ultrametric \(d\) on a set of points \(V\)
is \emph{non-trivial} if the following conditions hold.
\begin{enumerate}
\item For every non-empty set \(S \subseteq V\), there is a pair of points \(i, j \in S\) such
that \(d(i, j) \ge \size{S} - 1\).\label{def:non-trivial:spreading}

\item For any \(t\) if \(S_t\) is an equivalence class of \(V\) under the relation 
\(i \sim j\) iff \(d(i, j) \le t\), then \(\max_{i, j \in S_t} d(i, j) \le \size{S_t} - 1\).
\label{def:non-trivial:hereditary}
\end{enumerate}
\end{definition}

Note that for an equivalence class \(S_t\) where \(d(i, j) \le t\) for
every \(i, j \in S_t\) it follows from Condition~\ref{def:non-trivial:spreading} that
\(t \ge \size{S_t} - 1\). Thus in the case when \(t = \size{S_t} - 1\) the two conditions
imply that the maximum distance between any two points in \(S\) is \(t\) and that there is a pair
\(i, j \in S\) for which this maximum is attained. The following lemma shows
that non-trivial ultrametrics behave well under restrictions to equivalence classes
\(S_t\) of the form \(i \sim j\) iff \(d(i, j) \le t\). 

\begin{lemma}\label{lem:restriction}
 Let \(d\) be a non-trivial ultrametric on \(V\) and let \(S_t \subseteq V\) be an equivalence
 class under the relation \(i \sim j\) iff \(d(i, j) \le t\). Then
 \(d\) restricted to \(S_t\) is a non-trivial ultrametric on \(S_t\). 
\end{lemma}

\begin{proof}
 Clearly \(d\) restricted to \(S_t\) is an ultrametric on \(S_t\) and so we need to establish that
 it satisfies Conditions~\ref{def:non-trivial:spreading} and \ref{def:non-trivial:hereditary} of 
 Definition~\ref{def:non-trivial}. Let \(S \subseteq S_t\)
 be any set. Since \(d\) is a non-trivial ultrametric on \(V\) it follows that 
 there is a pair \(i, j \in S\) with \(d(i, j) \ge \size{S} - 1\), and so \(d\) restricted to
 \(S_t\) satisfies Condition~\ref{def:non-trivial:spreading}. 
 
 If \(S'_r\) is an equivalence class in \(S_t\) under the relation \(i \sim j\) iff \(d(i, j) \le r\) then clearly
 \(S'_r = S_t\) if \(r > t\). Since \(d\) is a non-trivial ultrametric on \(V\), it follows that
 \(\max_{i, j \in S'_r} d(i, j) = \max_{i, j \in S_t} d(i, j) \le \size{S_t} - 1 = \size{S'_r} - 1\).
 Thus we may assume that \(r \le t\). Consider an \(i \in S'_r\) and let \(j \in V\) be such that \(d(i, j) \le r\). 
 Since \(r \le t\) and \(i \in S_t\),
 it follows that \(j \in S_t\) and so \(j \in S'_r\). In other words \(S'_r\) is an equivalence class in 
 \(V\) under the relation \(i \sim j\) iff \(d(i, j) \le r\). Since \(d\) is an ultrametric on \(V\) it follows that
 \(\max_{i, j \in S'_r} d(i, j) \le \size{S'_r} - 1\). Thus \(d\) restricted to \(S_t\) satisfies
 Condition~\ref{def:non-trivial:hereditary}.
\end{proof}

The intuition behind the two conditions in Definition~\ref{def:non-trivial} is as follows.
Condition~\ref{def:non-trivial:spreading} imposes a certain lower bound by
ruling out trivial ultrametrics where, e.g., \(d(i, j) = 1\) for every distinct pair
\(i, j \in V\).
On the other hand Condition~\ref{def:non-trivial:hereditary} discretizes  
and imposes an upper bound on \(d\) by
restricting its range to the set \(\{0, 1, \dots, n-1\}\) (see Lemma~\ref{lem:discrete}).
This rules out the other spectrum of triviality where for example \(d(i, j) = n\) for every 
distinct pair \(i, j \in V\) with \(\size{V} = n\). 

\begin{lemma}\label{lem:discrete}
 Let \(d\) be a non-trivial ultrametric on the set \(V\) as in Definition~\ref{def:non-trivial}. Then
 the range of \(d\) is contained in the set \(\{0, 1, \dots, n - 1\}\) with \(\size{V} = n\).
\end{lemma}

\begin{proof}
 We will prove this by induction on \(\size{V}\). The base case when \(\size{V} = 1\) is trivial. 
 Therefore, we now assume that \(\size{V} > 1\). By Condition~\ref{def:non-trivial:spreading} 
 there is a pair \(i, j \in V\) such that \(d(i, j) \ge n - 1\). Let 
 \(t = \max_{i, j \in V} d(i, j)\), then the only equivalence class under the relation \(i \sim j\) iff
 \(d(i, j) \le t\) is \(V\). By Condition~\ref{def:non-trivial:hereditary} it follows that
 \(\max_{i, j \in V} d(i, j) = t = n - 1\). Let \(V_1, \dots V_m\) denote the set of equivalence classes
 of \(V\) under the relation \(i \sim j\) iff \(d(i, j) \le n - 2\). Note that \(m > 1\) as there
 is a pair \(i, j \in V\) with \(d(i, j) = n -1\), and therefore each \(V_l
 \subsetneq V\). By Lemma~\ref{lem:restriction}, \(d\) restricted to each of these \(V_i\)'s is 
 a non-trivial ultrametric on those sets. The claim then follows immediately: 
 for any \(i, j \in V\) either \(i, j \in V_l\) for some \(V_l\) in which case by the induction hypothesis
 \(d(i, j) \in \left\{0, 1, \dots, \size{V_l} - 1\right\}\), or \(i \in V_l\) and \(j \in V_{l'}\) for \(l \neq l'\) in 
 which case \(d(i, j) = n - 1\).
\end{proof}

\section{Ultrametrics and Hierarchical Clusterings}
\label{sec:main}

We start with the following easy lemma about the lowest common ancestors of subsets
of \(V\) in a hierarchical clustering \(T\) of \(V\). 

\begin{lemma}\label{lem:lca}
 Let \(S \subseteq V\) with \(\size{S} \ge 2\). If \(r = \lca(S)\) then there is a pair 
 \(i, j \in S\) such that \(\lca(i, j) = r\).
\end{lemma}

\begin{proof}
We will proceed by induction on \(|S|\). If \(|S| = 2\) then the claim is trivial and so we may
assume \(|S| > 2\). Let \(i \in S\) be an arbitrary point
and let \(r' = \lca(S\setminus \{i\})\). We claim that \(r = \lca(i, r')\). Clearly the 
subtree rooted at \(\lca(i, r')\) contains \(S\) and since \(T[r]\) is the smallest such tree it 
follows that \(r \in T[\lca(i, r')]\). 

Conversely, \(T[r]\) contains \(S\setminus \{i\}\) and so
\(r' \in T[r]\) and since \(i \in T[r]\), it follows that \(\lca(i, r') \in T[r]\).
Thus we conclude that \(r = \lca(i, r')\). 

If \(\lca(i, r') = r'\), then we are done by the induction hypothesis. Thus we may assume that
\(i \notin T[r']\). Consider any \(j \in S\) such that \(j \in T[r']\). Then we have that
\(\lca(i, j) = r\) as \(\lca(i, r') = r\) and \(j \in T[r']\) and \(i \notin T[r']\).
\end{proof}

We will now show that non-trivial ultrametrics on \(V\) as in Definition~\ref{def:non-trivial} are exactly
those that are induced by hierarchical clusterings on \(V\) under cost function~\eqref{cost}.
The following lemma shows the forward direction: the ultrametric \(d_T\) induced by any hierarchical
clustering \(T\) is non-trivial.

\begin{lemma}\label{lem:hierarchy-nontrivial}
 Let \(T\) be a hierarchical clustering on \(V\) and let \(d_T\) be the ultrametric on \(V\) induced by it.
 Then \(d_T\) is non-trivial.
\end{lemma}

\begin{proof}
Let \(S \subseteq V\) be arbitrary and \(r = \lca(S)\), then \(T[r]\) has at least \(|S|\) 
leaves. By Lemma~\ref{lem:lca} there must be a pair \(i, j \in S\) such that 
\(r = \lca(i, j)\) and so \(d_T(i, j) \ge |S| - 1\). This satisfies Condition~\ref{def:non-trivial:spreading}
of non-triviality. 

For any \(t\), let \(S_t\) be a non-empty equivalence class under
the relation \(i \sim j\) iff \(d_T(i, j) \le t\). Since \(d_T\) satisfies Condition~\ref{def:non-trivial:spreading}
it follows that \(\size{S_t} - 1 \le t\). Let us assume for the sake of contradiction that
there is a pair \(i, j \in S_t\) such that \(d_T(i, j) > \size{S_t} - 1\). Let \(r = \lca(S_t)\); 
using the definition of \(d_T\) it follows that \(t + 1 \ge \size{\leaves \left( T[r]\right)} > \size{S_t}\) since \(i, j \in S_t\). 
Let \(k \in \leaves\left(T[r]\right) \setminus S_t\) be an arbitrary point, then for every \(l \in S_t\) it follows
that \(d_T(k, l) \le \size{\leaves(T[r])} - 1 \le t\) since the subtree rooted at \(r\) contains both \(k\) and \(l\).
This is a contradiction to \(S_t\) being an equivalence class under \(i \sim j\) iff \(d_T(i, j) \le t\)
since \(k \notin S_t\). Thus \(d_T\) also satisfies Condition~\ref{def:non-trivial:hereditary} of Definition~\ref{def:non-trivial}.
\end{proof}

The following crucial lemma shows the converse: every non-trivial ultrametric
on \(V\) is realized by a hierarchical clustering \(T\) of \(V\).

\begin{lemma}\label{lem:bijection-ultrametric}
For every non-trivial ultrametric \(d\) on \(V\) there is a hierarchical clustering
\(T\) on \(V\) such that for any pair \(i, j \in V\) we have
\begin{align*}
d_T(i, j) = \size{\leaves(T[\lca\left(i, j\right)])} - 1 = d(i, j).
\end{align*}
Moreover this hierarchy can be constructed in time \(O\left(n^3\right)\) by Algorithm~\ref{algo:buildtree} 
where \(\size{V} = n\).
\end{lemma}

\begin{proof}
 The proof is by induction on \(n\). The base case when \(n = 1\) is straightforward. We now suppose that
 the statement is true for sets of size \(< n\). Note that \(i \sim j\) iff \(d(i, j) \le n - 2\) is an
 equivalence relation on \(V\) and thus partitions \(V\) into \(m\) equivalence classes \(V_1, \dots, V_m\).
 We first observe that \(m > 1\) since by Condition~\ref{def:non-trivial:spreading} there is a pair
 of points \(i, j \in V\) such that \(d(i, j) \ge n - 1\) and in particular 
 \(\size{V}_l < n\) for every \(l \in \{1, \dots, m\}\). 
 By Lemma~\ref{lem:restriction}, \(d\) restricted to any \(V_l\) is a non-trivial ultrametric on \(V_l\)
 and there is a pair of points \(i, j \in V_l\) such that \(d(i, j) = \size{V_l} - 1\) by 
 Conditions~\ref{def:non-trivial:spreading} and \ref{def:non-trivial:hereditary}.
 Therefore by the induction hypothesis we construct trees \(T_1, \dots, T_m\) such that 
 for every \(l \in \{1, \dots, m\}\) we have \(\leaves(T_l) = V_l\).
 Further for any pair of points \(i, j \in V_l\) for some \(l \in \{1, \dots, m\}\), we also have
 \( d(i, j) = d_{T_l}(i, j)\).
 
 We construct the tree \(T\) as follows: we first add a root \(r\) and then connect the root \(r_l\)
 of \(T_l\) to \(r\) for every \(l \in \{1, \dots, m\}\).
 Consider a pair of points \(i, j \in V\). If \(i, j \in V_l\) for some \(l \in \{1, \dots, m\}\)
 then we are done since \(d_{T_l}(i, j) = d_T(i, j)\) as \(\lca(i, j) \in T_l\).
 If \(i \in V_l\) and \(j \in V_{l'}\) for some \(l \neq l'\) then \(d(i, j) = n - 1\) since 
 \(d(i, j) \ge n - 1\) by definition of the equivalence relation and the range of \(d\) lies in \(\{0, 1, \dots, n - 1\}\)
 by Lemma~\ref{lem:discrete}. Moreover \(i\) and \(j\) are leaves in \(T_l\) and \(T_{l'}\)
 respectively, and thus by construction of \(T\) we have \(\lca(i, j) = r\), i.e., 
 \(d_T(i, j) = n - 1\) and so the claim follows. 
 Algorithm~\ref{algo:buildtree} simulates this inductive argument
 can be easily implemented to run in time \(O\left(n^3\right)\).
\end{proof}

Lemmas~\ref{lem:hierarchy-nontrivial} and \ref{lem:bijection-ultrametric} together imply
the following corollary about the equivalence of hierarchical clusterings 
and non-trivial ultrametrics.

\begin{corollary}\label{cor:equivalent}
 There is a bijection between the set of hierarchical clusterings \(T\) on \(V\) and 
 the set of non-trivial ultrametrics \(d\) on \(V\) satisfying the following conditions.
 \begin{enumerate}
  \item For every hierarchical clustering \(T\) on \(V\), there is a non-trivial ultrametric \(d_T\) defined
  as \(d_T(i, j) \coloneqq \size{\leaves{T[\lca(i, j)]}} - 1\) for every \(i, j \in V\).
  
  \item For every non-trivial ultrametric \(d\) on \(V\), there is a hierarchical clustering \(T\) on \(V\)
  such that for every \(i, j \in V\) we have \(\size{\leaves{T[\lca(i, j)]}} - 1 = d(i, j)\).
 \end{enumerate}
Moreover this bijection can be computed in \(O(n^3)\) time, where \(\size{V} = n\).
\end{corollary}

\begin{algorithm}[!htb]\label{algo:buildtree}
\DontPrintSemicolon 
\KwIn{Data set \(V\) of \(n\) points, non-trivial ultrametric \(d: V \times V\to \R_{\ge 0}\)}
\KwOut{Hierarchical clustering \(T\) of \(V\) with root \(r\)} 
\(r\gets\) arbitrary choice of designated root in \(V\)\\
\(X \gets \{r\}\)\\
\(E \gets \emptyset\)\\
\uIf{\(n = 1\)}{
  \(T \gets (X, E)\) \\  
  \Return{\(r, T\)}\;
  }
\Else{
 Partition \(V\) into \(\{V_1, \dots V_m\}\) under the equivalence relation 
 \(i \sim j\) iff \(d(i, j) < n - 1\)\\
 \For{\(l \in \{1, \dots, m\}\)}{
  Let \(r_l, T_l\) be output of Algorithm~\ref{algo:buildtree} on \(V_l, \restr{d}{V_l}\)\\
  \(X \gets X \cup V(T_l)\)\\
  \(E \gets E \cup \{r, r_l\}\)\\
  }
  \(T \gets (X, E)\) \\
  \Return{\(r, T\)}
}
\caption{Hierarchical clustering of \(V\) from non-trivial ultrametric}
\end{algorithm}

Therefore to find the hierarchical clustering of minimum cost, it suffices to minimize \(\sprod{\kappa}{d}\) over
non-trivial ultrametrics \(d: V \times V \to \{0, \dots, n-1\}\), where \(V\) is the data set. Note that
the cost of the ultrametric \(d_T\) corresponding to a tree \(T\)
is an affine offset of \(\cost(T)\). In particular, we have \(\sprod{\kappa}{d_T} = \cost(T) - \sum_{\{i, j\} \in E(K_n)}
\kappa(i, j)\).

A natural approach is to formulate this problem as an Integer Linear Program (ILP)
and then study LP or SDP relaxations of it. 
We consider the following ILP for this problem that is motivated by 
\cite{di2015finding}. 
We have the variables \(x^{1}_{ij}, \dots, x^{n-1}_{ij}\) for every distinct pair \(i, j \in V\) 
with \(x^{t}_{ij} = 1\) if and only
if \(d(i, j) \ge t\). For any positive integer \(n\), let \([n] \coloneqq \{1, 2, \dots, n\}\).

\begin{align}\label{ilp:ultrametric}
\tag{ILP-ultrametric}
\min \qquad & \sum_{t = 1}^{n-1} \sum_{\{i, j\} \in E(K_n)} \kappa (i, j) x^t_{ij}\\ 
 \text{s.t.} \qquad & x^t_{ij} \ge x^{t+1}_{ij} \quad \qquad \forall i, j \in V, t \in [n - 2] \label{eq:nonincreasing}\\
 \qquad & x^t_{ij} + x^t_{jk} \ge x^t_{ik}  \quad \qquad \forall i, j, k \in V, t \in [n-1] \label{eq:triangle}\\
 \qquad & \sum_{i, j \in S} x^{t}_{ij} \ge 2 \quad \qquad \forall t\in [n-1], S \subseteq V, |S| = t + 1\label{eq:spreading}\\
 \qquad & \sum_{i, j \in S} x^{|S|}_{ij} \le \size{S}  \left(\sum_{i, j \in S} x^t_{ij} + 
 \sum_{\substack{i \in S\\ j \notin S}}\left(1 - x^t_{ij}
 \right)\right) \forall t\in [n-1], S \subseteq V \label{eq:hereditary}\\
 \qquad & x^t_{ij} = x^t_{ji} \quad\qquad \forall i, j \in V, t \in [n-1] \label{eq:symmetry}\\
 \qquad & x^t_{ii} = 0 \quad\qquad \forall i \in V, t \in [n-1] \label{eq:identity}\\
 \qquad & x^t_{ij} \in \{0, 1\} \quad\qquad \forall i, j \in V, t \in [n-1]
\end{align}

Constraints~\eqref{eq:nonincreasing} and \eqref{eq:identity}
follow from the interpretation of the variables \(x^t_{ij}\): 
if \(d(i, j) \ge t\), i.e., \(x^t_{ij} = 1\) then clearly \(d(i, j) \ge t-1\) and so
\(x^{t-1}_{ij} = 1\). Furthermore, for any \(i \in V\) we have \(d(i, i) = 0\) and 
so \(x^t_{ii} = 0\) for every \(t \in [n-1]\).
Note that constraint~\eqref{eq:triangle} is the same as the \emph{strong triangle inequality}
(Definition~\ref{def:ultrametric}) since
the variables \(x^t_{ij}\) are in \(\{0, 1\}\).
Constraint~\ref{eq:symmetry} ensures that
the ultrametric is symmetric. 
Constraint~\ref{eq:spreading} ensures the ultrametric satisfies 
Condition~\ref{def:non-trivial:spreading} of 
non-triviality: for every \(S\subseteq V\)
of size \(t + 1\) we know that there must be points \(i, j \in S\) such that
\(d(i, j) = d(j, i) \ge t\) or in other words \(x^t_{ij}  = x^t_{ji} = 1\). 
Constraint~\ref{eq:hereditary} ensures that the ultrametric 
satisfies Condition~\ref{def:non-trivial:hereditary} of non-triviality. To see this
note that the constraint is active only when \(\sum_{i, j \in S} x^t_{ij} = 0\)
and \(\sum_{i \in S, j \notin S} (1 - x^t_{ij}) = 0\). 
In other words \(d(i, j) \le t - 1\) for every \(i, j \in S\) and \(S\) is a maximal such 
set since if \(i \in S\) and \(j \notin S\) then \(d(i, j) \ge t\). Thus \(S\)
is an equivalence class under the relation \(i \sim j\) iff \(d(i, j) \le t-1\) and so
for every \(i, j \in S\) we have \(d(i, j) \le \size{S} - 1\) or equivalently \(x^{\size{S}}_{ij} = 0\).
The ultrametric \(d\) represented 
by a feasible solution \(x^t_{ij}\) is given by 
\(d(i, j) = \sum_{t = 1}^{n-1} x^t_{ij}\). 

\begin{definition}\label{def:cliques}
For any \(\left\{x^t_{ij} \mid t \in [n-1], i, j \in V\right\}\) let 
\(E_t\) be defined as \(E_t \coloneqq \left\{\{i, j\} \mid x^t_{ij} = 0\right\}\). Note that
if \(x^t_{ij}\) is feasible for \ref{ilp:ultrametric} then 
\(E_{t} \subseteq E_{t+1}\) for any \(t\) since \(x^t_{ij} \ge x^{t+1}_{ij}\).
The sets \(\{E_t\}_{t = 1}^{n-1}\) induce a natural sequence of graphs \(\{G_t\}_{t=1}^{n-1}\)
where \(G_t = (V, E_t)\) with \(V\) being the data set.
\end{definition}

For a fixed \(t \in \{1, \dots, n-1\}\) it is instructive to study the combinatorial properties
of the so called \emph{layer-\(t\) problem}, 
where we restrict ourselves to the constraints corresponding to that particular \(t\)
and drop constraints~\eqref{eq:nonincreasing} and \eqref{eq:hereditary} since they
involve different layers in their expression.

\begin{align}
 \label{ilp:layer} \tag{ILP-layer} \min \qquad &\sum_{\{i, j\} \in E(K_n)} \kappa (i, j) x^t_{ij} \\
 \text{s.t.} \qquad & x^t_{ij} + x^t_{jk} \ge x^t_{ik} \quad\qquad \forall i, j, k \in V \label{eq:layer:triangle}\\
       \qquad & \sum_{i, j \in S} x^t_{ij} \ge 2 \quad \qquad \forall S \subseteq V, \size{S} = t + 1 \label{eq:layer:nontrivial}\\
       \qquad & x^t_{ij} = x^t_{ji}  \quad \qquad \forall i, j \in V \label{eq:layer:symmetry} \\
       \qquad & x^t_{ii} = 0 \quad \qquad \forall i \in V \label{eq:layer:identity}\\
       \qquad & x^t_{ij} \in \{0, 1\} \quad\qquad \forall i, j \in V
\end{align}

The following lemma provides a combinatorial characterization of feasible solutions to the layer-\(t\)
problem.

\begin{lemma}\label{lem:cliques}
 Let \(G_t = (V, E_t)\) be the graph as in Definition~\ref{def:cliques} corresponding to a solution \(x^t_{ij}\) to the
 layer-\(t\) problem \ref{ilp:layer}. Then \(G_t\) is a disjoint union of cliques of size
 \(\le t\). Moreover this exactly characterizes all feasible solutions of \ref{ilp:layer}.
\end{lemma}

\begin{proof}
 We first note that \(G_t = (V, E_t)\) must be a disjoint
 union of cliques since if \(\{i, j\} \in E_t\) and \(\{j, k\} \in E_t\)
 then \(\{i, k\} \in E_t\) since \(x^t_{ik} \le x^t_{ij} + x^t_{jk} = 0\)
 due to constraint~\eqref{eq:layer:triangle}. Suppose there is a
 clique in \(G_t\) of size \(> t\). Choose a subset \(S\)
 of this clique of size \(t + 1\). Then \(\sum_{i, j \in S} x^t_{ij} = 0\) which violates
 constraint~\eqref{eq:layer:nontrivial}.
 
 Conversely, let \(E_t\) be a subset of edges such that \(G_t = (V, E_t)\) is a disjoint 
 union of cliques of size \(\le t\). Let \(x^t_{ij} = 0\)
 if \(\{i, j\} \in E_t\) and \(1\) otherwise. Clearly \(x^t_{ij} = x^t_{ji}\) by 
 definition. Suppose \(x_{ij}^t\) violates constraint~\eqref{eq:layer:triangle}, so that
 there is a pair \(i, j, k \in V\) such that \(x^t_{ik} = 1\) but \(x^t_{ij} = x^t_{jk} = 0\). However this implies
 that \(G_t\) is not a disjoint union of cliques since \(\{i, j\}, \{j, k\} \in E_t\)
 but \(\{i, k\}\notin E_t\). Suppose \(x^t_{ij}\) violates constraint~\eqref{eq:layer:nontrivial}
 for some set \(S\) of size \(t+1\). Therefore for every \(i, j \in S\), we have \(x^t_{ij} = 0\) since
 \(x^t_{ij} = x^t_{ji}\) for every \(i, j\in V\) and so \(S\) must
 be a clique of size \(t+1\) in \(G_t\) which is a contradiction.
\end{proof}

By Lemma~\ref{lem:cliques} the layer-\(t\) problem 
is to find a subset \(\overline{E}_t \subseteq E(K_n)\) of minimum weight under \(\kappa\), 
such that the complement graph \(G_t = (V, E_t)\) is a disjoint union of
cliques of size \(\le t\). Note that this implies that the number of 
components in the complement graph is \(\ge \lceil n/t\rceil\).The converse however, is not necessarily true:
when \(t = n-1\) then the layer \(t\)-problem is the minimum (weighted) cut problem whose partitions 
may have size larger than \(1\).
Our algorithmic approach is to solve an LP relaxation of \ref{ilp:ultrametric} and
then round the solution to obtain a feasible solution to \ref{ilp:ultrametric}.
The rounding however proceeds iteratively in a layer-wise manner and so we need
to make sure that the rounded solution satisfies the inter-layer constraints~\eqref{eq:nonincreasing}
and \eqref{eq:hereditary}. The following lemma gives a combinatorial characterization of
solutions that satisfy these two constraints. 

\begin{lemma}\label{lem:inter-layer}
 For every \(t \in [n-1]\), let \(x^t_{ij}\) be feasible for the layer-\(t\) problem \ref{ilp:layer}.
 Let \(G_t = (V, E_t)\) be the graph as in Definition~\ref{def:cliques} corresponding to \(x^t_{ij}\), so
 that by Lemma~\ref{lem:cliques}, \(G_t\) is a disjoint union of cliques \(K^t_1, \dots, K^t_{l_t}\)
 each of size at most \(t\). Then \(x^t_{ij}\) is feasible for \ref{ilp:ultrametric} if and only if
 the following conditions hold.
 \begin{description}
 \item[Nested cliques\label{lem:inter-layer:cond1}] For any \(s \le t\) every clique \(K^s_p\) 
 for some \(p \in [l_s]\) in \(G_s\) is a subclique
 of some clique \(K^t_q\) in \(G_t\) where \(q \in [l_t]\).
 
 \item[Realization\label{lem:inter-layer:cond2}] If \(\size{K^t_p} = s\) for some \(s \le t\), 
 then \(G_s\) contains \(K^t_p\) as a component clique, i.e.,
 \(K^s_q = K^t_p\) for some \(q \in [l_s]\). 
\end{description}
 \end{lemma}

\begin{proof}
 Since \(x^t_{ij}\) is feasible for the layer-\(t\) problem \ref{ilp:layer} it is feasible for \ref{ilp:ultrametric}
 if and only if it satisfies constraints~\eqref{eq:nonincreasing} and \eqref{eq:hereditary}.
 The solution \(x^t_{ij}\) satisfies constraint~\eqref{eq:nonincreasing} if and only if \(E_t \subseteq E_{t + 1}\)
 by definition and so Condition~\ref{lem:inter-layer:cond1} follows. 
 
 Let us now assume that \(x^t_{ij}\) is feasible for \ref{ilp:ultrametric}, 
 so that by the above argument Condition~\ref{lem:inter-layer:cond1} is satisfied.
 Note that every clique \(K^t_p\) in the clique decomposition of \(G_t\) corresponds to
 an equivalence class \(S_t\) under the relation \(i \sim j\) iff \(x^t_{ij} = 0\).
 Moreover, by Lemma~\ref{lem:cliques} we have \(\size{S_t} \le t\).
 Constraint~\eqref{eq:hereditary} implies that \(x^{\size{S_t}}_{ij} = 0\) for every \(i, j \in S_t\). 
 In other words, if 
 \(\size{S_t} = s \le t\), then \(x^{s}_{ij} = 0\) for every \(i, j \in S_t\)
 and so \(S_t\) is a subclique of some clique
 \(K^s_q\) in the clique decomposition of \(G_s\). However by Condition~\ref{lem:inter-layer:cond1},
 \(K^s_q\) must be a subclique of a clique \(K^t_{p'}\) in the clique decomposition of \(G_t\), 
 since \(s \le t\).
 However, as \(K^t_p \cap K^t_{p'} = S_t\) and the clique decomposition decomposes \(G_t\) into a
 disjoint union of cliques, it follows that \(S_t \subseteq K^s_q \subseteq K^t_{p'} = K^t_p = S_t\)
 and so \(K^s_q = K^t_p\). Therefore Condition~\ref{lem:inter-layer:cond2} is satisfied.

 Conversely, suppose that \(x^t_{ij}\)
 satisfies Conditions~\ref{lem:inter-layer:cond1} and \ref{lem:inter-layer:cond2}, so 
 that by the argument in the paragraph above \(x^t_{ij}\) satisfies constraint~\eqref{eq:nonincreasing}.
 Let us assume for the sake of contradiction that for a set \(S \subseteq V\) and a \(t \in [n-1]\)
 constraint~\eqref{eq:hereditary} is violated, i.e.,
 \begin{align*}
  \sum_{i, j \in S} x^{|S|}_{ij} > \size{S}  \left(\sum_{i, j \in S} x^t_{ij} + 
 \sum_{\substack{i \in S\\ j \notin S}}\left(1 - x^t_{ij}
 \right)\right).
 \end{align*}
Since \(x^t_{ij} \in \{0, 1\}\) it follows that \(x^t_{ij} = 0\) for every \(i, j \in S\) and 
\(x^t_{ij} = 1\) for every \(i \in S, j \notin S\) so that \(S\) 
is a clique in \(G_t\). Note that \(\size{S} < t\) since \(\sum_{i, j \in S} x^{\size{S}}_{ij} > 0\).
This contradicts Condition~\ref{lem:inter-layer:cond2} however, since \(S\) is clearly not a clique
in \(G_{\size{S}}\).
\end{proof}

The combinatorial interpretation of the individual layer-\(t\) problems allow us to 
simplify the formulation of \ref{ilp:ultrametric} by replacing the constraints for
sets of a specific size (constraint~\eqref{eq:spreading}) by a global constraint about all sets
(constraint~\eqref{eq:equiv-spreading}). 

\begin{lemma}\label{lem:equiv-spreading}
We may replace constraint~\eqref{eq:spreading} of \ref{ilp:ultrametric}
by the following equivalent constraint

\begin{align}
 \sum_{j \in S} x^t_{ij} \ge |S| - t \qquad \forall t\in [n-1], S \subseteq V, i \in S \label{eq:equiv-spreading}.
\end{align}
\end{lemma}

\begin{proof}
 Let \(x^t_{ij}\) be a feasible solution to \ref{ilp:ultrametric}.
 Note that if \(|S| \le t\) then the constraints are redundant since \(x^t_{ij} \in \{0, 1\}\).
 Thus we may assume that \(|S| > t\) and let \(i\) be any vertex in \(S\). Let us suppose
 for the sake of a contradiction 
 that \(\sum_{j \in S} x^t_{ij} < |S|-t\). This implies that there is a \(t\) sized
 subset \(S' \subseteq S \setminus \{i\}\) such that for every \(j \in S'\) we have \(x^t_{ij'} = 0\).
 In other words \(\{i, j'\}\) is an edge in \(G_t = (V, E_t)\) for every
 \(j' \in S'\) and since \(G_t\) is a disjoint union of cliques (constraint~\eqref{eq:triangle}), 
 this implies
 the existence of a clique of size \(t + 1\). Thus by Lemma~\ref{lem:cliques},
 \(x^t_{ij}\) could not have been a feasible solution to \ref{ilp:ultrametric}.
 
 Conversely, suppose \(x^t_{ij}\) is feasible for the modified ILP where constraint~\eqref{eq:spreading}
 is replaced by constraint~\eqref{eq:equiv-spreading}. Then again  
 \(G_t = (V, E_t)\) is a disjoint union of cliques since \(x^t_{ij}\) satisfies 
 constraint~\eqref{eq:triangle}. Assume for contradiction that constraint~\eqref{eq:spreading}
 is violated: there is a set \(S\) of size \(t + 1\) such that
 \(\sum_{i, j \in S} x^t_{ij} < 2\). Note that this implies that 
 \(\sum_{i, j} x^t_{ij} = 0\) since \(x^t_{ij} = x^t_{ji}\) for every \(i, j \in V\) and
 \(t \in [n-1]\). Fix any \(i \in S\), then \(\sum_{j \in S}
 x^t_{ij} < 1 = |S| - t\) since \(x^t_{ij} = x^t_{ji}\) by constraint~\eqref{eq:symmetry}, a
 violation of constraint~\eqref{eq:equiv-spreading}. 
 Thus \(x^t_{ij}\) is feasible for \ref{ilp:ultrametric} since it
 satisfies every other constraint by assumption.
\end{proof}

\section{Rounding an LP relaxation}\label{sec:lp-rounding}
In this section we consider the following natural LP relaxation for \ref{ilp:ultrametric}. 
We keep the variables \(x^t_{ij}\) for every \(t \in [n-1]\) and \(i, j \in V\) but
relax the integrality constraint on the variables as well as drop constraint~\eqref{eq:hereditary}.

\begin{align}\label{lp:ultrametric}
 \tag{LP-ultrametric}
 \min \qquad &\sum_{t=1}^{n-1} \sum_{\{i, j\} \in E(K_n)} \kappa(i, j) x^t_{ij} \\
 \text{s.t.}
 \qquad &x^t_{ij} \ge x^{t+1}_{ij} \quad \qquad \forall i, j \in V,  t \in [n - 2] \label{lp:layer} \\
 \qquad &x^t_{ij} + x^t_{jk} \ge x^t_{ik} \quad \qquad \forall i, j, k \in V, t \in [n-1] \label{lp:triangle}\\
 \qquad & \sum_{j \in S} x^t_{ij} \ge \size{S} - t \quad \qquad \forall t\in [n-1], S \subseteq V, i \in S \label{lp:spreading}\\
 \qquad & x^t_{ij} = x^t_{ji} \quad \qquad \forall i, j \in V, t \in [n-1] \label{lp:symmetry}\\
 \qquad & x^t_{ii} = 0 \quad \qquad \forall i, j \in V, t \in [n-1] \label{lp:identity} \\
 \qquad & 0 \le x^t_{ij} \le 1 \quad \qquad\forall i, j \in V, t \in [n-1]
\end{align}
A feasible solution \(x^t_{ij}\) to \ref{lp:ultrametric}
induces a sequence \(\{d_t\}_{t \in [n-1]}\) of distance metrics over \(V\) defined as 
\(d_t(i, j) \coloneqq x^t_{ij}\).
Constraint~\ref{lp:spreading} enforces an additional structure on this metric:
informally points in a ``large enough'' subset \(S\) should be spread apart according
to the metric \(d_t\). Metrics of type \(d_t\) are called \emph{spreading metrics}
and were first studied in \cite{even1999fast, even2000divide} in relation to graph partitioning problems.
The following lemma gives a technical interpretation of spreading metrics (see, e.g.,
\cite{even1999fast, even2000divide, krauthgamer2009partitioning}); we include a proof for completeness.

\begin{lemma}\label{lem:spreading}
 Let \(x^t_{ij}\) be feasible for \ref{lp:ultrametric} and for a fixed \(t \in [n-1]\), let \(d_t\) be 
 the induced spreading metric. Let \(i \in V\) be an arbitrary vertex and let \(S \subseteq V\) be a 
 set with \(i \in S\) such that \(\size{S} > (1 + \varepsilon) t\) for some \(\varepsilon > 0\). Then 
 \(\max_{j \in S} d_t(i, j) > \frac{\varepsilon}{1 + \varepsilon}\).
\end{lemma}

\begin{proof}
For the sake of a contradiction suppose that for every \(j \in S\) we have
\(d_t(i, j) = x^t_{ij} \le \frac{\varepsilon}{1 + \varepsilon}\). 
This implies that \(x^t_{ij}\) violates constraint~\eqref{lp:spreading} leading to a contradiction:
\begin{align*}
 \sum_{j \in S} x^t_{ij} \le \frac{\varepsilon}{1 + \varepsilon}\size{S} < \size{S} - t,
\end{align*}
where the last inequality follows from \(\size{S} > (1 + \varepsilon)t\).
\end{proof}

The following lemma shows that we can optimize over \ref{lp:ultrametric} in polynomial time.

\begin{lemma}\label{lem:polytime}
An optimal solution to \ref{lp:ultrametric} can be computed in time polynomial in \(n\) and
\(\log\left(\max_{i, j} \kappa(i, j)\right)\).
\end{lemma}

\begin{proof}
We argue in the standard fashion via the application of the Ellipsoid method 
 (see e.g., \cite{schrijver1998theory}). As such it suffices to verify that the encoding length
 of the numbers is small (which is indeed the case here) and that the constraints can be separated in polynomial time
 in the size of the input, i.e., in \(n\) and the logarithm of the absolute value of the 
 largest coefficient.
 Since constraints of type \eqref{lp:layer}, \eqref{lp:triangle}, \eqref{lp:symmetry}, and
 \eqref{lp:identity} are 
 polynomially many in \(n\), we only need to check
 separation for constraints of type \eqref{lp:spreading}. Given a claimed solution 
 \(x^t_{ij}\) we can check constraint~\eqref{lp:spreading}
 by iterating over all \(t \in [n-1]\), vertices \(i \in V\), and sizes \(m\) of the set \(S\)
 from \(t + 1\) to \(n\). 
 For a fixed \(t, i\), and set size \(m\) sort the vertices in \(V \setminus \{i\}\) in 
 increasing order of distance from
 \(i\) (according to the metric \(d_t\)) and let \(\overline{S}\) be the first \(m\) vertices in this ordering.  
 If \(\sum_{j \in \overline{S}} x^t_{ij} < m - t\) then clearly \(x^t_{ij}\) is not feasible for \ref{lp:ultrametric}, 
 so we may assume that \(\sum_{j \in \overline{S}} x^t_{ij} \ge m - t\).
 Moreover this is the only set to check: for any set \(S \subseteq V\) containing \(i\) such that \(\size{S} = m\), 
 \(\sum_{j \in S} x^t_{ij} \ge \sum_{j \in \overline{S}} x^t_{ij} \ge m - t\).
 Thus for a fixed \(t \in [n-1]\), \(i \in V\) and set size \(m\), it suffices to check 
 that \(x^t_{ij}\) satisfies constraint~\eqref{lp:spreading} for this subset \(\overline{S}\).
\end{proof}

From now on we will simply refer to a feasible solution to \ref{lp:ultrametric} by the sequence
of spreading metrics \(\{d_t\}_{t \in [n-1]}\) it induces. The following definition introduces 
the notion of an open ball \(\ball{i}{r}{t}{U}\) of radius \(r\) centered at \(i \in V\) 
according to the metric \(d_t\) and restricted to the set \(U \subseteq V\).

\begin{definition}\label{def:ball}
Let \(\left\{d_t \mid t \in [n-1]\right\}\) be the sequence of spreading metrics 
feasible for \ref{lp:ultrametric}. Let \(U \subseteq V\) be an arbitrary subset of \(V\). For
 a vertex \(i \in U\), \(r \in \R\), and \(t \in [n-1]\) we define the \emph{open ball \(\ball{i}{r}{t}{U}\)}
 of radius \(r\) centered at \(i\) as 
 \begin{align*}
  \ball{i}{r}{t}{U} \coloneqq \left\{j \in U \mid d_t(i, j) < r\right\} \subseteq U.
\end{align*}
If \(U = V\) then we denote \(\ball{i}{r}{t}{U}\) simply by \(\ball{i}{r}{t}\).
\end{definition}

\begin{remark}
 For every pair \(i, j \in V\) we have  \(d_t(i, j) \ge d_{t+1}(i, j)\) by constraint~\eqref{lp:layer}. 
 Thus for any subset \(U \subseteq V\), \(i \in U\), \(r \in \mathbb{R}\), and \(t \in [n-2]\), it holds
 \(\ball{i}{r}{t}{U} \subseteq \ball{i}{r}{t+1}{U}\).
\end{remark}
 
To round \ref{lp:ultrametric} to get a feasible solution for \ref{ilp:ultrametric},
we will use the technique of \emph{sphere growing} which was introduced in \cite{leighton1988approximate} 
to show an \(O(\log{n})\) approximation for the maximum multicommodity flow problem.
Recall from Lemma~\ref{lem:cliques} that a feasible solution to \ref{ilp:layer} consists of 
a decomposition of the graph \(G_t\) into a set of disjoint cliques of size at most \(t\).
One way to obtain such a decomposition is to choose an arbitrary vertex, grow a ball around this vertex until
the expansion of this ball 
is below a certain threshold, chop off this ball and declare it as a partition and then recurse on the remaining
vertices. This is the main idea behind sphere growing, and the parameters are chosen 
depending on the constraints of the specific problem (see, e.g., \cite{garg1996approximate, even1999fast, charikar2003clustering} 
for a few representative applications of this technique).
The first step is to associate to every ball \(\ball{i}{r}{t}{U}\)
a volume \(\vol{\ball{i}{r}{t}{U}}\) and a boundary \(\partial \ball{i}{r}{t}{U}\) so that its expansion is defined.
For any \(t \in [n-1]\) and \(U \subseteq V\) we denote by \(\gamma^U_t\) the
value of the layer-\(t\) objective for solution \(d_t\) restricted to the set \(U\), i.e., 
\begin{align*}
\gamma^U_t \coloneqq \sum_{\substack{i, j \in U \\ i < j}} \kappa(i, j) d_t(i, j).
\end{align*} 
When \(U = V\) we refer to \(\gamma^U_t\) simply by \(\gamma_t\). Since \(\kappa : V \times V \to \R_{\ge 0}\),
it follows that \(\gamma^U_t \le \gamma_t\) for any \(U \subseteq V\). 
We are now ready to define the volume, boundary, and expansion of a ball \(\ball{i}{r}{t}{U}\).
We use the definition of \cite{even1999fast} modified for restrictions to arbitrary subsets \(U \subseteq V\).

\begin{definition}\cite{even1999fast}\label{def:volume}
 Let \(U\) be an arbitrary subset of \(V\). For a vertex \(i \in U\), 
 radius \(r \in \R_{\ge 0}\), and \(t \in [n-1]\), let \(\ball{i}{r}{t}{U}\) be
 the ball of radius \(r\) as in Definition~\ref{def:ball}. Then we define its
 \emph{volume} as
 \begin{align*}
  \vol{\ball{i}{r}{t}{U}} \coloneqq \frac{\gamma^U_t}{n\log{n}} + 
  \sum_{\substack{j, k \in \ball{i}{r}{t}{U}\\j < k}} \kappa(j, k) d_t(j, k) + 
  \sum_{\substack{j \in \ball{i}{r}{t}{U}\\ k \notin \ball{i}{r}{t}{U} \\ k \in U}} \kappa(j, k) \left(r - d_t(i, j)\right).
 \end{align*}
 The \emph{boundary} of the ball \(\partial\ball{i}{r}{t}{U}\) is the partial derivative of volume with respect to
 the radius:
 \begin{align*}
  \partial \ball{i}{r}{t}{U} \coloneqq \frac{\partial\vol{\ball{i}{r}{t}{U}}}{\partial r} = \sum_{\substack{j \in \ball{i}{r}{t}{U}\\
  k \notin \ball{i}{r}{t}{U} \\ k \in U}} \kappa(j, k).
 \end{align*}
The \emph{expansion} \(\phi(\ball{i}{r}{t}{U})\) of the ball \(\ball{i}{r}{t}{U}\) is defined as 
the ratio of its boundary to its volume, i.e.,
\begin{align*}
 \phi\left(\ball{i}{r}{t}{U}\right) \coloneqq \frac{\partial{\ball{i}{r}{t}{U}}}{\vol{\ball{i}{r}{t}{U}}}.
\end{align*}
\end{definition}

The following lemma shows that the volume of a ball \(\ball{i}{r}{t}{U}\) is differentiable
with respect to \(r\) in the interval \((0, \Delta]\) except at finitely many points 
(see e.g., \cite{even1999fast}). 

\begin{lemma}\label{lem:differentiable} 
Let \(\ball{i}{r}{t}{U}\) be the ball corresponding to a
set \(U \subseteq V\), vertex \(i \in U\), radius \(r \in \R\) and
 \(t \in [n-1]\). Then \(\vol{\ball{i}{r}{t}{U}}\) is differentiable with respect
 to \(r\) in the interval \((0, \Delta]\) except at finitely many points.
\end{lemma}

\begin{proof}
Note that for any fixed 
 \(U \subseteq V\), \(\vol{\ball{i}{r}{t}{U}}\) is a monotone non-decreasing function in \(r\) 
 since for a pair \(j, k \in U\) such that \(j \in \ball{i}{r}{t}{U}\) and \(k \notin
 \ball{i}{r}{t}{U}\) we have \(r - d_t(i, j) \le d_t(j, k)\) otherwise \(r - d_t(i, j) > d_t(j, k)\)
 so that \(r > d_t(i, j) + d_t(j, k) \ge d_t(i, k)\), a contradiction to the fact that \(k \notin \ball{i}{r}{t}{U}\).
 Therefore adding the vertex \(k\) to the ball centered at \(i\) is only going to increase its volume as
 \(r - d_t(i, j) \le d_t(j, k)\) (see Definition~\ref{def:ball}).
 Thus \(\vol{\ball{i}{r}{t}{U}}\) is differentiable with respect to \(r\) in the interval \((0, \Delta]\) except
 at finitely many points which correspond to a new vertex from \(U\) being added to the ball.
\end{proof}

\begin{algorithm}[!htbp]\label{algo:rounding}
\DontPrintSemicolon 
\KwIn{Data set \(V\), \(\{d_t\}_{t\in[n-1]} : V \times V\), \(\varepsilon > 0\), \(\kappa: V \times V \to \R_{\ge 0}\)}
\KwOut{A solution set of the form \(\left\{ x^t_{ij} \in \{0, 1\} \mid t \in \left[\left\lfloor \frac{n-1}{1 + \varepsilon}\right\rfloor
\right], i, j \in V\right\}\) } 
\(m_{\varepsilon} \gets \left\lfloor\frac{n-1}{1 + \varepsilon}\right\rfloor\) \label{algo:line:solveLP}\\
\(t \gets m_{\varepsilon}\)\\
\(\mathcal{C}_{t + 1} \gets \{V\}\)\label{algo:line:initialization}\\
\(\Delta \gets \frac{\varepsilon}{1 + \varepsilon}\)\\
\While{\(t \ge 1\)} { \label{algo:line:tloop}
  \(\mathcal{C}_t \gets \emptyset\)\\
  \For{\(U \in \mathcal{C}_{t + 1}\)}{ \label{algo:line:iterate}
   \If{\(\size{U} \le (1 + \varepsilon)t\)} { \label{algo:line:smallball}
	\(\mathcal{C}_t \gets \mathcal{C}_t \cup \{U\}\) \label{algo:line:leftover}\\
	Go to line~\ref{algo:line:iterate} \\
    }
    \While{\(U \neq \emptyset\)}{
	    Let \(i\) be arbitrary in \(U\)\\
	    Let \(r \in (0, \Delta]\) be s.t. \(\phi\left(\ball{i}{r}{t}{U}\right) \le \frac{1}{\Delta} 
	    \log{\left(\frac{\vol{\ball{i}{\Delta}{t}{U}}}{\vol{\ball{i}{0}{t}{U}}} \right)}\)\label{algo:line:subball}\\
	    \(\mathcal{C}_t \gets \mathcal{C}_t \cup \{\ball{i}{r}{t}{U}\}\) \label{algo:line:addsubball}\\
	    \(U \gets U \setminus \ball{i}{r}{t}{U}\)\\ \label{algo:line:subtractball}
	    }
  }
    \(x^t_{ij} = 1\) if \(i \in U_1 \ \in \mathcal{C}_t\), \(j \in U_2 \in C_t\) and \(U_1 \neq U_2\), else \(x^t_{ij} = 0\)\\ \label{algo:line:clique}
  \(t \gets t - 1\)\\
}
\Return{\(\left\{x^t_{ij} \mid t \in [m_{\varepsilon}], i, j \in V\right\}\)}\;
\caption{Iterative rounding algorithm to find a low cost ultrametric}
\end{algorithm}

The following theorem establishes that the rounding procedure of Algorithm~\ref{algo:rounding}
ensures that the cliques in \(\mathcal{C}_t\) are ``small'' and that the cost of the edges 
removed to form them are not too high. It also shows that Algorithm~\ref{algo:rounding} can be
implemented to run in time polynomial in \(n\). 

\begin{theorem}\label{thm:approx}
Let \(m_\varepsilon \coloneqq 
\left\lfloor \frac{n-1}{1 + \varepsilon}\right\rfloor\) as in Algorithm~\ref{algo:rounding}
 and let \(\left\{x^t_{ij} \mid t \in [m_\varepsilon], i, j \in V\right\}\) be the output of 
 Algorithm~\ref{algo:rounding} run on a feasible solution \(\{d_t\}_{t\in [n-1]}\) of 
 \ref{lp:ultrametric} and any choice of \(\varepsilon \in (0, 1)\). For any 
 \(t \in \left[m_\varepsilon\right]\), we have that \(x^t_{ij}\) is feasible
 for the layer-\(\left\lfloor\left(1 + \varepsilon\right)t\right\rfloor\) problem \ref{ilp:layer} and there is a constant
 \(c(\varepsilon) > 0\) depending only on \(\varepsilon\) such that
 \begin{align*}
    \sum_{\{i, j\} \in E(K_n)} \kappa(i, j) x^t_{ij} \le c(\varepsilon) (\log{n})\gamma_t.
 \end{align*}
 Moreover, Algorithm~\ref{algo:rounding} can be implemented to run in time polynomial in \(n\).
\end{theorem}

\begin{proof}
We first show that for a fixed \(t\), the constructed solution \(x^t_{ij}\) is feasible for the 
 layer-\(\lfloor (1 + \varepsilon)t\rfloor\) problem \ref{ilp:layer}. Let \(\mathcal{C}_t\) be as
 in Algorithm~\ref{algo:rounding} so that \(x^t_{ij} = 1\) if \(i, j\) belong to different
 sets in \(\mathcal{C}_t\) and \(x^t_{ij} = 0\) otherwise. Let \(G_t = (V, E_t)\) be 
 as in Definition~\ref{def:cliques} corresponding to \(x^t_{ij}\).
 Note that for any \(t \in [m_\varepsilon]\), every \(V_i \in \mathcal{C}_t\) is a 
 clique in \(G_t\) by construction (line~\ref{algo:line:clique}) 
 and for every distinct pair \(V_i, V_j \in \mathcal{C}_t\) we have 
 \(V_i \cap V_j = \emptyset\) (lines~\ref{algo:line:addsubball} and
 \ref{algo:line:subtractball}).
 Therefore by Lemma~\ref{lem:cliques},
 it suffices to prove that for any \(V_i \in \mathcal{C}_t\), it holds \(\size{V_i} \le \lfloor (1 + \varepsilon)t\rfloor\).
 If \(V_i\) is added to \(\mathcal{C}_t\) in line~\ref{algo:line:leftover} then there is
 nothing to prove. 
 
 Thus let us assume that \(V_i\) is of the form \(\ball{i}{r}{t}{U}\) for some \(U \subseteq V\) as in 
 line~\ref{algo:line:subball} so that \(\phi\left(\ball{i}{r}{t}{U}\right) \le \frac{1}{\Delta}
 \log{\left(\frac{\vol{\ball{i}{\Delta}{t}{U}}}{\vol{\ball{i}{0}{t}{U}}}\right)}\). Note that by 
 Lemma~\ref{lem:spreading} it suffices to show that there is such an \(r \in (0, \Delta]\). This property follows from
 the rounding scheme due to \cite{even1999fast} as we will explain now. 
 
 By Lemma~\ref{lem:differentiable} \(\vol{\ball{i}{r}{t}{U}}\) is differentiable 
 everywhere in the interval \((0, \Delta]\) except at finitely many points \(X\). 
 Let the set of discontinuous points be \(X = \{x_1, x_2, \dots, x_{k-1}\}\) with 
 \(x_0 = 0 < x_1 < x_2 \dots x_{k-1} < x_k=\Delta\).
 We claim that there must be an \(r \in (0, \Delta] \setminus
 X\) such that \(\phi\left(\ball{i}{r}{t}{U}\right) \le \frac{1}{\Delta}
 \log{\left(\frac{\vol{\ball{i}{\Delta}{t}{U}}}{\vol{\ball{i}{0}{t}{U}}}\right)}\).
 Let us assume for the sake of a contradiction that for every 
 \(r \in \left(0, \Delta\right]\setminus X\) we have
 \(\phi\left(\ball{i}{r}{t}{U} \right) > 
 \frac{1}{\Delta}\log{\left(\frac{\vol{\ball{i}{\Delta}{t}{U}}}{\vol{\ball{i}{0}{t}{U}}}\right)}\).
 However integrating both sides from \(0\) to \(\Delta\) results in a contradiction:
 \begin{align}
  \int_{r = 0}^{\Delta} \phi\left(\ball{i}{r}{t}{U}\right) dr &= 
  \int_{r=0}^{\Delta} \frac{\partial \ball{i}{r}{t}{U}}{\vol{\ball{i}{r}{t}{U}}}dr \\
  &= \sum_{i=1}^k \int_{r=x_{i-1}}^{x_i} \frac{\partial \ball{i}{r}{t}{U}}{\vol{\ball{i}{r}{t}{U}}}dr\\
  &= \sum_{i=1}^k \int_{r=x_{i-1}}^{x_i} \frac{d\left(\vol{\ball{i}{r}{t}{U}}\right)}{\vol{\ball{i}{r}{t}{U}}} \\
  &\le \log{\vol{\ball{i}{\Delta}{t}{U}}} - \log{\vol{\ball{i}{0}{t}{U}}} \label{eq:monotonic}\\
  &= \int_{r=0}^{\Delta} \frac{1}{\Delta}\log{\left(\frac{\vol{\ball{i}{\Delta}{t}{U}}}{\vol{\ball{i}{0}{t}{U}}}\right)}dr,
 \end{align}
 where line~\ref{eq:monotonic} follows since \(f\) is monotonic increasing.
 For any \(t \in [m_\varepsilon]\) the set \(\mathcal{C}_t\) is a disjoint partition of \(V\) 
 with balls of the form \(\ball{i}{r}{t'}{U}\) for some  
 \(t' \ge t\) and \(U \subseteq U_l \in \mathcal{C}_{t' + 1}\): this is easily seen by induction since \(\mathcal{C}_{m_\varepsilon + 1}\)
 is initialized as \(V\). 
 Further, a cluster \(V_i\) is added to \(\mathcal{C}_t\) either in line~\ref{algo:line:addsubball}
 in which case it is a ball of the form \(\ball{i}{r}{t}{U}\) for some \(U \in \mathcal{C}_{t + 1}\), 
 \(i \in U\), and \(r \in \R\) or it is added in line~\ref{algo:line:leftover} in which 
 case it must have been a ball \(\ball{i'}{r'}{t'}{U}\) for some \(t' > t\), \(U \subseteq U_l \in \mathcal{C}_{t' + 1}\), 
 \(i' \in V\), and \(r' \in \R\). 
 Note that for any \(t' \ge t\) and \(U \subseteq V\), it holds \(\gamma^U_{t'} \le \gamma^U_{t}\) since for 
 every pair \(i, j \in V\) we have \(\kappa(i, j) \ge 0\) and \(d_t(i, j) \ge d_{t'}(i, j)\) 
 because of constraint~\eqref{lp:layer}. Moreover, for any subset \(U \subseteq V\) we have 
 \(\gamma^U_t \le \gamma_t\) since \(\kappa, d_t \ge 0\).
 
 We claim that for any
 \(t \in \left[m_\varepsilon\right]\) the total volume of the balls in \(\mathcal{C}_t\) 
 is at most \(\left(2 + \frac{1}{\log{n}}\right)\gamma_t\). First note that the affine term
 \(\frac{\gamma^U_{t'}}{n\log{n}}\)
 in the volume of a ball \(\ball{i}{r}{t'}{U}\) in \(\mathcal{C}_t\) is upper bounded by \(\frac{\gamma_t}{n\log{n}}\) and appears at most \(n\) 
 times. Next we claim that the contribution to the total volume from the term involving the edges 
 inside and crossing a ball \(\ball{i}{r}{t'}{U} \in \mathcal{C}_t\)
 is at most \(2\gamma_t\). This is because the balls are disjoint, \(r - d_{t'}(i, k) \le d_{t'}(j, k)
 \le d_t(j, k)\) for the crossing edges of a ball \(\ball{i}{r}{t'}{U} \in \mathcal{C}_t\) and 
 a crossing edge contributes to the volume of at most \(2\) balls in \(\mathcal{C}_t\).
 Note that for any \(U \subseteq V\), \(i \in U\), and \(r \in \R_{\ge 0}\) we have 
 \(\vol{\ball{i}{r}{t}{U}} \in \left[\frac{\gamma^U_t}{n\log{n}}, \left(1 + \frac{1}{n\log{n}}\right)\gamma^U_t\right]\).
 Using this observation and the stopping condition of line~\ref{algo:line:subball} it follows that
 \begin{align*}
  \sum_{\{i, j\} \in E(K_n)} \kappa(i, j) x^t_{ij}
  &= \sum_{\substack{\{i, j\} \in E(K_n): \\ i, j \text{ separated in } \mathcal{C}_t}} \kappa(i, j)\\
  &= \underbrace{\frac{1}{2}\sum_{\substack{\ball{i}{r}{t'}{U} \in \mathcal{C}_t: \\ t' \ge t \\ U \subseteq U_l \in \mathcal{C}_{t' + 1}}}
  \sum_{\substack{j \in \ball{i}{r}{t'}{U} \\ k \notin \ball{i}{r}{t'}{U}}} \kappa(j, k)}_{\text{Since \(\kappa\) is symmetric}}\\
  &= \frac{1}{2} \sum_{\substack{\ball{i}{r}{t'}{U} \in \mathcal{C}_t: \\ t' \ge t \\ U \subseteq U_l \in \mathcal{C}_{t'+1}}} \partial\ball{i}{r}{t'}{U}\\
  &= \frac{1}{2} \sum_{\substack{\ball{i}{r}{t'}{U} \in \mathcal{C}_t: 
  \\ t'\ge t \\ U \subseteq U_l \in \mathcal{C}_{t'+1}}} \phi\left(\ball{i}{r}{t'}{U}\right) 
  \vol{\ball{i}{r}{t'}{U}}\\
  &\le 
  \sum_{\substack{\ball{i}{r}{t'}{U} \in \mathcal{C}_t:\\ t'\ge t\\ U \subseteq U_l \in \mathcal{C}_{t'+1}}}
  \frac{1}{2\Delta}\log{\left(\frac{\vol{\ball{i}{\Delta}{t'}{U}}}{\vol{\ball{i}{0}{t'}{U}}} \right)}
   \vol{\ball{i}{r}{t'}{U}} \\
  &\le \frac{1}{2\Delta}\underbrace{\left(\log\left(n\log{n} + 1\right)\right)}_{\text{via interval bounds}}
  \sum_{\substack{\ball{i}{r}{t'}{U}\in \mathcal{C}_t: \\ t' \ge t\\ U \subseteq U_l \in \mathcal{C}_{t'+1}}}\vol{\ball{i}{r}{t'}{U}} \\
  &\le \frac{1 + \varepsilon}{2\varepsilon}\left(\log\left(n\log{n} + 1\right)\right)
  \underbrace{\left(2 + \frac{1}{\log{n}}\right)\gamma_t}_{\substack{\text{contribution of affine term \(\le \frac{\gamma_t}{\log n}\)} 
   \\ \text{contribution of edge terms \(\le 2\gamma_t\)}}}\\
  &\le c(\varepsilon)(\log{n})\gamma_t,
 \end{align*}
for some constant \(c(\varepsilon) > 0\) depending only on \(\varepsilon\).

For the run time of Algorithm~\ref{algo:rounding} note that the loop in line~\ref{algo:line:tloop} runs for at most
\(n-1\) steps, while the loop in line~\ref{algo:line:iterate} runs for at most \(n\) steps. For a set 
\(U \subseteq V\), to compute the ball 
\(\ball{i}{r}{t}{U}\) of least radius \(r\) such that 
\(\phi\left(\ball{i}{r}{t}{U}\right) \le \frac{1}{\Delta}
\log{\left(\frac{\vol{\ball{i}{\Delta}{t}{U}}}{\vol{\ball{i}{0}{t}{U}}} \right)}\),
sort the vertices in \(U\setminus \{i\}\) in increasing order of distance from \(i\) according to \(d_t\).
Let the vertices in \(U\setminus \{i\}\) in this sorted order be \(\left\{j_1, \dots, j_{\size{U}-1}\right\}\). Then it 
suffices to check the expansion of the balls \(\{i\}\) and \(\{i\}\cup \{j_1, \dots, j_k\}\) for every 
\(k \in \left[\size{U}-1\right]\). It is straightforward to see that all the other steps in Algorithm~\ref{algo:rounding} run in 
time polynomial in \(n\).
\end{proof}

\begin{remark} 
 A discrete version of the volumetric argument for region growing can be found in \cite{gupta2005lecture}.
\end{remark}

We are now ready to prove the main theorem showing that we can obtain a low cost non-trivial ultrametric 
from Algorithm~\ref{algo:rounding}.

\begin{theorem}\label{thm:main}
Let \(\{x^t_{ij} \mid t \in \left[m_\varepsilon\right], i,j \in V\}\) be the output of Algorithm~\ref{algo:rounding}
on an optimal solution \(\{d_t\}_{t\in [n-1]}\) of \ref{lp:ultrametric} for any choice of 
\(\varepsilon \in (0, 1)\).
Define the sequence \(\left\{y^t_{ij}\right\}\) for every \(t \in [n-1]\) and \(i, j \in V\) as
\begin{align*}
y^t_{ij} \coloneqq \begin{cases} 
		    x^{\lfloor t/(1 + \varepsilon) \rfloor}_{ij} \quad &\text{if  } t > 1 + \varepsilon\\
		    1 \quad &\text{if } t \le 1 + \varepsilon.
                   \end{cases}
\end{align*}
Then \(y^t_{ij}\) is feasible for \ref{ilp:ultrametric} and satisfies
\begin{align*}
\sum_{t=1}^{n-1} \sum_{\{i, j\} \in E(K_n)} \kappa(i, j) y^t_{ij} \le 
\left(2c(\varepsilon) \log{n}\right) \OPT
\end{align*}
where \(\OPT\) is the optimal solution to \ref{ilp:ultrametric} and \(c(\varepsilon)\) is the constant
in the statement of Theorem~\ref{thm:main}.
\end{theorem}

\begin{proof}
Note that by Theorem~\ref{thm:approx} for every \(t \in \left[m_\varepsilon\right]\),
\(x^t_{ij}\) is feasible for the layer-\(\lfloor(1 + \varepsilon)t\rfloor\) problem \ref{ilp:layer} and that 
there is a constant \(c(\varepsilon) > 0\) such that for every \(t \in \left[m_\varepsilon\right]\), 
we have \(\sum_{\{i,j\} \in E(K_n)} \kappa (i, j) x^t_{ij} \le
\left( c(\varepsilon)\log{n}\right) \gamma_t\).

Let \(y^t_{ij}\) be as in the statement of the theorem.
The graph \(G_t = (V, E_t)\) as in Definition~\ref{def:cliques} corresponding to
\(y^t_{ij}\) for \(t \le 1 + \varepsilon\) consists of isolated vertices, i.e., cliques of size \(1\):
By definition \(y^t_{ij}\) is feasible for the layer-\(t\) problem \ref{ilp:layer}. 
The collection \(\mathcal{C}_1\) corresponding to \(x^1_{ij}\) consists of cliques of 
size at most \(1 + \varepsilon\), however since \(0 < \varepsilon < 1\) it follows that the cliques in
\(\mathcal{C}_1\) are isolated vertices and so \(x^1_{ij} = 1\) for every 
\(\{i, j\} \in E(K_n)\). Thus \(\sum_{i, j} \kappa(i, j) y^t_{ij} = \sum_{i, j} \kappa(i, j) x^1_{ij}
\le \left(c(\varepsilon) \log{n}\right) \gamma_1\) for \(t \le 1 + \varepsilon\) by Theorem~\ref{thm:approx}.
Moreover for every \(t > 1 + \varepsilon\), we have
\(\sum_{i, j} \kappa (i, j) y^t_{ij} \le (c(\varepsilon) \log{n})\gamma_{\lfloor t/(1 + \varepsilon) \rfloor}\)
again by Theorem~\ref{thm:approx}.
We claim that \(y^t_{ij}\) is feasible for \ref{ilp:ultrametric}. 
The solution \(y^t_{ij}\) corresponds to the collection \(\mathcal{C}_{\lfloor\frac{t}{1 + \varepsilon}\rfloor}\)
for \(t > 1 + \varepsilon\) or to the collection \(\mathcal{C}_1\) for \(t \le 1 + \varepsilon\)
from Algorithm~\ref{algo:rounding}. For any \(t < m_\varepsilon\),
every ball \(\ball{i}{r}{t}{U} \in \mathcal{C}_t\) 
comes from the refinement of a ball \(\ball{i'}{r'}{t'}{U'}\)
for some \(i' \in V\), \(r' \ge r\), \(t' \ge t\) and \(U' \supseteq U\). 
Thus \(y^t_{ij}\) satisfies Condition~\ref{lem:inter-layer:cond1} of Lemma~\ref{lem:inter-layer}.
On the other hand line~\ref{algo:line:smallball} ensures that if \(\size{\ball{i}{r}{t}{U}} = 
\left\lfloor (1 + \varepsilon)s\right\rfloor\)
for some \(U \subseteq V\) and \(s < t\) then \(\ball{i}{r}{t}{U}\) also appears as a ball in \(\mathcal{C}_s\).
Therefore \(y^t_{ij}\) also satisfies Condition~\ref{lem:inter-layer:cond2} of Lemma~\ref{lem:inter-layer}
and so is feasible for \ref{ilp:ultrametric}. The cost of \(y^t_{ij}\) is at most

\begin{align*}
 \sum_{t=1}^{n-1} \sum_{\{i, j\} \in E(K_n)} \kappa(i, j) y^t_{ij} 
 &\le \left(c(\varepsilon) \log{n}\right)\left(\gamma_1 + \sum_{t = 2}^{n-1} \gamma_{\lfloor t/(1 + \varepsilon) \rfloor}\right) \\
 &\le 2c(\varepsilon)\log{n} \sum_{t=1}^{n-1} \gamma_t \\
 &\le 2c(\varepsilon) \log{n} \OPT,
\end{align*}
where we use the fact that \(\sum_{t=1}^{n-1}\gamma_t = \OPT(LP) \le \OPT\) since \ref{lp:ultrametric}
is a relaxation of \ref{ilp:ultrametric}.
\end{proof}

\begin{algorithm}[!htbp]\label{algo:obtain-hierarchy}
\DontPrintSemicolon 
\KwIn{Data set \(V\) of \(n\) points, similarity function \(\kappa: V \times V \to \R_{\ge 0}\)}
\KwOut{Hierarchical clustering of \(V\)}
Solve \ref{lp:ultrametric} to obtain optimal sequence of spreading metrics \(\left\{d_t \mid d_t: V \times V \to 
[0, 1]\right\}\)\\
Fix a choice of \(\varepsilon \in (0, 1)\)\\
\(m_\varepsilon \gets \left\lfloor \frac{n-1}{1 + \varepsilon}\right\rfloor\)\\
Let \(\left\{x^t_{ij}\mid t \in [m_\varepsilon]\right\}\) be the output of Algorithm~\ref{algo:rounding} on \(V, \kappa, \{d_t\}_{t \in [n-1]}\)\\
Let \(y^t_{ij} \coloneqq \begin{cases} x^{\lfloor t/(1 + \varepsilon)\rfloor}_{ij} &\text{ if } t > 1 + \varepsilon
\\ 1 \qquad&\text{ if } t \le 1 + \varepsilon\end{cases}\)  for every \(t \in [n-1], i, j \in E(K_n)\)\\
\(d(i, j) \gets \sum_{t = 1}^{n-1} y^t_{ij}\) for every \(i, j \in E(K_n)\)\\
\(d(i, i) \gets 0\) for every \(i \in V\)\\
Let \(r, T\) be the output of Algorithm~\ref{algo:buildtree} on \(V, d\)\\
\Return{\(r, T\)}\;
\caption{Hierarchical clustering of \(V\) for cost function~\eqref{cost}}
\end{algorithm}

Theorem~\ref{thm:main} implies the following corollary where we put everything together to
obtain a hierarchical clustering of \(V\) in time polynomial in \(n\) with \(\size{V} = n\).
Let \(\mathcal{T}\) denote the set of all possible hierarchical clusterings of \(V\).

\begin{corollary}
Given a data set \(V\) of \(n\) points and a similarity function \(\kappa: V \times V \to \R_{\ge 0}\),
Algorithm~\ref{algo:obtain-hierarchy} returns a hierarchical clustering \(T\) of \(V\) satisfying
\begin{align*}
 \cost(T) \le O\left(\log{n}\right) \min_{T' \in \mathcal{T}} \cost(T').
\end{align*}
Moreover Algorithm~\ref{algo:obtain-hierarchy} runs in time polynomial in \(n\) and \(\log\left({\max_{i, j \in V}\kappa(i, j)}\right)\).
\end{corollary}

\begin{proof}
Let \(\widehat{T}\) be the optimal hierarchical clustering according to cost function~\eqref{cost}.
By Corollary~\ref{cor:equivalent} and Theorem~\ref{thm:main} we can find a hierarchical 
clustering \(T\) satisfying 
\begin{align*}
\sum_{\{i, j\} \in E(K_n)} \kappa(i, j) (\size{\leaves(T[\lca(i, j)])} - 1)
\le O(\log n) \left(\sum_{\{i, j\} \in E(K_n)} \kappa(i, j) 
\left(\size{\leaves(\widehat{T}[\lca(i, j)])} - 1\right)\right).
\end{align*}
Let \(K \coloneqq \sum_{\{i, j\} \in E(K_n)}\kappa(i, j)\). Then it follows from the above
expression that \(\cost(T) \le O(\log{n})\cost(\widehat{T}) - O(\log{n}) K + K \le O(\log{n}) \cost(\widehat{T})\). 

We can find an optimal solution to \ref{lp:ultrametric} due to Lemma~\ref{lem:polytime} using the 
Ellipsoid algorithm in time polynomial in \(n\) and \(\log\left(\max_{i, j \in V}\kappa(i, j)\right)\). 
Algorithm~\ref{algo:rounding} runs in time polynomial in \(n\) due to Theorem~\ref{thm:approx}.
Finally, Algorithm~\ref{algo:buildtree} runs in time \(O\left(n^3\right)\) due to Lemma~\ref{lem:bijection-ultrametric}.
\end{proof}

\section{Generalized Cost Function}\label{sec:f-lp-rounding}
In this section we study the following natural generalization of cost function~\eqref{cost}
also introduced by \cite{DBLP:conf/stoc/Dasgupta16} where the distance between the two points 
is scaled by a function \(f: \R_{\ge 0} \to \R_{\ge 0}\), i.e.,
\begin{align}\label{fcost}
 \cost_f(T) \coloneqq \sum_{\{i, j\} \in E(K_n)} \kappa(i, j) f\left(\size{\leaves{T[\lca(i, j)]}}\right).
\end{align}
In order that cost function~\eqref{fcost} makes sense, \(f\) should be strictly increasing 
and satisfy \(f(0) = 0\). Possible choices for \(f\) could be 
\(\left\{x^2, e^x - 1, \log(1 + x)\right\}\).
The top-down heuristic in \cite{DBLP:conf/stoc/Dasgupta16} finds the optimal hierarchical clustering
up to an approximation factor of \(c_n \log{n}\) with \(c_n\) being defined as
\begin{align*}
 c_n \coloneqq 3\alpha_n \max_{1 \le n' \le n} \frac{f(n')}{f\left(\lceil n'/3 \rceil\right)}
\end{align*}
and where \(\alpha_n\) is the approximation factor from the Sparsest Cut algorithm used. 

A naive approach to solving this problem using the ideas of Algorithm~\ref{algo:rounding} would
be to replace the objective function of \ref{ilp:ultrametric} by 
\begin{align*}
\sum_{\{i, j\} \in E(K_n)} \kappa(i, j)f\left(\sum_{t=1}^{n-1} x^t_{ij}\right).
\end{align*}
This makes the corresponding analogue of \ref{lp:ultrametric} non-linear however, 
and for a general \(\kappa\) and \(f\) it is not
clear how to compute an optimum solution in polynomial time. One possible solution is to 
assume that \(f\) is convex and use the Frank-Wolfe algorithm to 
compute an optimum solution. That still leaves the problem of
how to relate \(f\left(\sum_{t=1}^{n-1} x^t_{ij}\right)\) to \(\sum_{t=1}^{n-1}f\left(x^t_{ij}\right)\)
as one would have to do to get a corresponding version of Theorem~\ref{thm:main}. The following
simple observation provides an alternate way of tackling this problem.

\begin{observation}\label{obs:function-ultrametric}
 Let \(d : V \times V \to \R\) be an ultrametric and \(f: \R_{\ge 0} \to \R_{\ge 0}\) 
 be a strictly increasing function such that \(f(0) = 0\).
 Define the function \(f(d): V \times V \to \R\) as \(f(d)(i, j) \coloneqq f(d(i, j))\). 
 Then \(f(d)\) is also an ultrametric on \(V\).
\end{observation}

Therefore by Corollary~\ref{cor:equivalent} to find a minimum cost hierarchical clustering \(T\)
of \(V\) according to the cost function~\eqref{fcost}, it suffices to minimize 
\(\sprod{\kappa}{d}\) where \(d\) is the \(f\)-image of a non-trivial ultrametric 
as in Definition~\ref{def:non-trivial}. The following lemma lays down the analogue of
Conditions~\ref{def:non-trivial:spreading} and \ref{def:non-trivial:hereditary} from
Definition~\ref{def:non-trivial} that the \(f\)-image of a non-trivial ultrametric satisfies. 

\begin{lemma}\label{lem:f-image-non-trivial}
 Let \(f: \R_{\ge 0} \to \R_{\ge 0}\) be a strictly increasing function satisfying \(f(0) = 0\).
 An ultrametric \(d\) on \(V\) is the \(f\)-image of a non-trivial ultrametric on \(V\) iff
 \begin{enumerate}
  \item\label{def:f-image-non-trivial:spreading} for every non-empty set \(S \subseteq V\), 
  there is a pair of points \(i, j \in S\)
  such that \(d(i, j) \ge f\left(\size{S} - 1\right)\), 
  
  \item\label{def:f-image-non-trivial:hereditary} for any \(t\) if \(S_t\) is an equivalence class
  of \(V\) under the relation \(i \sim j\) iff \(d(i, j) \le t\), then \(\max_{i, j \in S_t} d(i, j)
  \le f\left(\size{S_t} - 1\right)\).
 \end{enumerate}
\end{lemma}

\begin{proof}
If \(d\) is the \(f\)-image of a non-trivial ultrametric \(d'\) on \(V\) then clearly \(d\) 
satisfies Conditions~\ref{def:f-image-non-trivial:spreading} and \ref{def:f-image-non-trivial:hereditary}.
Conversely, let \(d\) be an ultrametric on \(V\) satisfying Conditions~\ref{def:f-image-non-trivial:spreading}
and \ref{def:f-image-non-trivial:hereditary}. Note that \(f\) is strictly increasing and \(V\) is a finite
set and thus \(f^{-1}\) exists and is strictly increasing as well, with \(f^{-1}(0) = 0\). 
Define \(d'\) as \(d'(i, j) \coloneqq f^{-1}(d(i, j))\) for every \(i, j \in V\).
By Observation~\ref{obs:function-ultrametric} \(d'\) is an ultrametric on \(V\) 
satisfying Conditions~\ref{def:non-trivial:spreading}
and \ref{def:non-trivial:hereditary} of Definition~\ref{def:non-trivial} and so \(d'\) is a non-trivial ultrametric on \(V\).
\end{proof}

Lemma~\ref{lem:f-image-non-trivial} allows us to write the analogue of \ref{ilp:ultrametric} for finding the 
minimum cost ultrametric that is the \(f\)-image of a non-trivial ultrametric on \(V\). Note that by Lemma~\ref{lem:discrete} 
the range of such an ultrametric is the set \(\{f(0), f(1), \dots, f(n-1)\}\).
We have the binary variables \(x^t_{ij}\) for every distinct pair \(i, j \in V\) and \(t \in [n-1]\), 
where \(x^t_{ij} = 1\) if \(d(i, j) \ge f(t)\) and \(x^t_{ij} = 0\) if \(d(i, j) < f(t)\). 

\begin{align}\label{ilp:f-ultrametric}
\tag{f-ILP-ultrametric}
\min \qquad & \sum_{t = 1}^{n-1} \sum_{\{i, j\} \in E(K_n)} \kappa (i, j) \left(f(t) - f(t-1)\right) x^t_{ij}\\ 
 \text{s.t.} \qquad & x^t_{ij} \ge x^{t+1}_{ij} \quad \qquad \forall i, j \in V, t \in [n-2] \label{eq:f-nonincreasing}\\
 \qquad & x^t_{ij} + x^t_{jk} \ge x^t_{ik}  \quad \qquad \forall i, j, k \in V, t \in [n-1] \label{eq:f-triangle}\\
 \qquad & \sum_{i, j \in S} x^t_{ij} \ge 2 \quad \qquad \forall t \in [n-1], S \subseteq V, |S| = t + 1\label{eq:f-spreading}\\
 \qquad & \sum_{i, j \in S} x^{\size{S}}_{ij} \le \size{S}  \left(\sum_{i, j \in S} x^t_{ij} + 
 \sum_{\substack{i \in S \\ j \notin S}}\left(1 - x^t_{ij}
 \right)\right)  \forall t \in [n-1], S \subseteq V \label{eq:f-hereditary}\\
 \qquad & x^t_{ij} = x^t_{ji} \quad\qquad \forall i, j \in V, t \in [n-1] \label{eq:f-symmetry}\\
 \qquad & x^t_{ii} = 0 \quad \qquad \forall i \in V, t \in [n-1] \label{eq:f-identity}\\
 \qquad & x^t_{ij} \in \{0, 1\} \quad\qquad \forall i, j \in V, t \in [n-1]
\end{align}

If \(x^t_{ij}\) is a feasible solution to \ref{ilp:f-ultrametric} then 
the ultrametric represented by it is defined as 
\begin{align*}
 d(i, j) \coloneqq \sum_{t=1}^{n-1} (f(t) - f(t-1)) x^t_{ij}.
\end{align*}
Constraint~\eqref{eq:f-spreading} ensures that \(d\) satisfies Condition~\ref{def:f-image-non-trivial:spreading}
of Lemma~\ref{lem:f-image-non-trivial}, since for every \(S \subseteq V\) of size \(t + 1\) 
we have a pair \(i, j \in S\) such that \(d(i, j) \ge f(t)\). Similarly constraint
~\eqref{eq:f-hereditary} ensures that \(d\) satisfies Condition~\ref{def:f-image-non-trivial:hereditary} of
Lemma~\ref{lem:f-image-non-trivial} since it is active if and only if \(S\) is an equivalence class of \(V\)
under the relation \(i \sim j\) iff \(d(i, j) < f(t)\). In this case Condition~\ref{def:f-image-non-trivial:hereditary}
of Lemma~\ref{lem:f-image-non-trivial} 
requires \(\max_{i, j \in S} d(i, j) \le f\left(\size{S}-1\right)\) or in other words \(x^{\size{S}}_{ij} = 0\)
for every \(i, j \in S\). 

Similar to \ref{ilp:layer} we define an analogous \emph{layer-\(t\) problem} where we fix a choice of 
\(t \in [n-1]\) and drop the constraints that relate the different layers to each other. 

\begin{align}
 \label{ilp:f-layer} \tag{f-ILP-layer} \min \qquad &\sum_{\{i, j\} \in E(K_n)} \kappa (i, j) \left(f(t) - f(t-1)\right)x^t_{ij} \\
 \text{s.t.} \qquad & x^t_{ij} + x^t_{jk} \ge x^t_{ik} \quad\qquad \forall i, j, k \in V \label{eq:f-layer:triangle}\\
       \qquad & \sum_{i, j \in S} x^t_{ij} \ge 2 \quad \qquad \forall S \subseteq V, \size{S} = t + 1 \label{eq:f-layer:nontrivial}\\
       \qquad & x^t_{ij} = x^t_{ji}  \quad \qquad \forall i, j \in V\label{eq:f-layer:symmetry} \\
       \qquad & x^t_{ii} = 0 \quad \qquad \forall i \in V \label{eq:f-layer:identity}\\
       \qquad & x^t_{ij} \in \{0, 1\} \quad\qquad \forall i, j \in V
\end{align}
Note that \ref{ilp:f-ultrametric} and \ref{ilp:f-layer} differ from \ref{ilp:ultrametric} and \ref{ilp:layer}
respectively only in the objective function. Therefore Lemmas~\ref{lem:cliques} and \ref{lem:inter-layer}
also give a combinatorial characterization of the set of feasible solutions to \ref{ilp:f-layer} and 
\ref{ilp:f-ultrametric} respectively. 
Similarly, by Lemma~\ref{lem:equiv-spreading} we may replace
constraint~\eqref{eq:f-spreading} by the following equivalent constraint over all subsets of \(V\)

\begin{align*}
 \sum_{j \in S} x^t_{ij} \ge \size{S} - t \qquad \forall t \in [n-1], S \subseteq V, i \in S.
\end{align*}

This provides the analogue of \ref{lp:ultrametric} in which we drop constraint~\eqref{eq:f-hereditary} 
and enforce it in the rounding procedure. 

\begin{align}\label{lp:f-ultrametric}
 \tag{f-LP-ultrametric}
 \min \qquad &\sum_{t=1}^{n-1} \sum_{\{i, j\} \in E(K_n)} \kappa(i, j)\left(f(t) - f(t-1)\right) x^t_{ij} \\
 \text{s.t.}
 \qquad &x^t_{ij} \ge x^{t+1}_{ij} \quad \qquad \forall i, j \in V,  t \in [n - 2] \label{lp:f-layer} \\
 \qquad &x^t_{ij} + x^t_{jk} \ge x^t_{ik} \quad \qquad \forall i, j, k \in V, t \in [n-1] \label{lp:f-triangle}\\
 \qquad & \sum_{j \in S} x^t_{ij} \ge \size{S} - t \quad \qquad \forall t\in [n-1], S \subseteq V, i \in S \label{lp:f-spreading}\\
 \qquad & x^t_{ij} = x^t_{ji} \quad \qquad \forall i, j \in V, t \in [n-1] \label{lp:f-symmetry}\\
 \qquad & x^t_{ii} = 0 \quad \qquad \forall i \in V, t \in [n-1] \label{lp:f-identity}\\
 \qquad & 0 \le x^t_{ij} \le 1 \quad \qquad\forall i, j \in V, t \in [n-1]
\end{align}

Since \ref{lp:f-ultrametric} differs from \ref{lp:ultrametric} only in the objective function, it follows from
Lemma~\ref{lem:polytime} that an optimum solution to \ref{lp:f-ultrametric} can be computed in time polynomial
in \(n\). As before, a feasible solution \(x^t_{ij}\) of \ref{lp:f-ultrametric} induces a sequence 
\(\{d_t\}_{t\in [n-1]}\) of spreading metrics on \(V\) defined as \(d_t(i, j) \coloneqq x^t_{ij}\).
Note that in contrast to the ultrametric \(d\), the spreading metrics \(\left\{d_t\right\}_{t \in [n-1]}\)
are independent of the function \(f\).

Let \(\ball{i}{r}{t}{U}\) be a ball of radius \(r\) centered at \(i \in U\) for some set \(U \subseteq V\)
as in Definition~\ref{def:ball}. 
For a subset \(U \subseteq V\), let \(\gamma^U_t\) be defined as before to be the value of the layer-\(t\) 
objective corresponding to a solution \(d_t\) of \ref{lp:f-ultrametric} 
restricted to \(U\), i.e., 
\begin{align*}
\gamma^U_t \coloneqq \sum_{\substack{i, j \in U \\ i < j}} \left(f(t) - f(t-1)\right) \kappa(i, j) d_t(i, j).
\end{align*}
As before, we denote \(\gamma^V_t\) by \(\gamma_t\).
We will associate a volume \(\vol{\ball{i}{r}{t}{U}}\) and a boundary
\(\partial \ball{i}{r}{t}{U}\) to the ball \(\ball{i}{r}{t}{U}\) as in Section~\ref{sec:lp-rounding}.

\begin{definition}\label{def:f-volume}
 Let \(U\) be an arbitrary subset of \(V\). For a vertex \(i \in U\), radius \(r \in \R_{\ge 0}\), 
 and \(t \in [n-1]\), let \(\ball{i}{r}{t}{U}\) be
 the ball of radius \(r\) as in Definition~\ref{def:ball}. Then we define its
 \emph{volume} as
 \begin{align*}
  \vol{\ball{i}{r}{t}{U}} \coloneqq \frac{\gamma^U_t}{n\log{n}} + \left(f(t) - f(t-1)\right)\left(
  \sum_{\substack{j, k \in \ball{i}{r}{t}{U}\\j < k}} \kappa(j, k) d_t(j, k) + 
  \sum_{\substack{j \in \ball{i}{r}{t}{U}\\ k \notin \ball{i}{r}{t}{U} \\ k \in U}} \kappa(j, k)
  \left(r - d_t(i, j)\right) \right).
 \end{align*}
 The \emph{boundary} of the ball \(\partial\ball{i}{r}{t}{U}\) is the partial derivative of volume with respect to
 the radius:
 \begin{align*}
  \partial \ball{i}{r}{t}{U} \coloneqq \left(f(t) - f(t-1)
  \right)\left(\frac{\partial\vol{\ball{i}{r}{t}{U}}}{\partial r}\right) =
  \left(f(t)-f(t-1)\right)\left(\sum_{\substack{j \in \ball{i}{r}{t}{U}\\
  k \notin \ball{i}{r}{t}{U} \\ k \in U}} \kappa(j, k)\right).
 \end{align*}
The \emph{expansion} \(\phi\left(\ball{i}{r}{t}{U}\right)\) of the ball \(\ball{i}{r}{t}{U}\) is defined as 
the ratio of its boundary to its volume, i.e.,
\begin{align*}
 \phi\left(\ball{i}{r}{t}{U}\right) \coloneqq \frac{\partial{\ball{i}{r}{t}{U}}}{\vol{\ball{i}{r}{t}{U}}}.
\end{align*}
\end{definition}

Note that the expansion \(\phi\left(\ball{i}{r}{t}{U}\right)\) of Definition~\ref{def:f-volume} is the same as
in Definition~\ref{def:volume} since the \(\left(f(t) - f(t-1)\right)\) term cancels out. 
Thus one could run Algorithm~\ref{algo:rounding} with the same notion of volume as in Definition~\ref{def:volume},
however in that case the analogous versions of Theorems~\ref{thm:approx} and \ref{thm:main}
do not follow as naturally.
The following is then a simple corollary of Theorem~\ref{thm:approx}.

\begin{corollary}\label{cor:f-approx}
 Let \(m_\varepsilon\coloneqq \left\lfloor \frac{n-1}{1 + \varepsilon}\right\rfloor\) as in 
Algorithm~\ref{algo:rounding}. 
 Let \(\left\{x^t_{ij} \mid t \in [n-1], i, j \in V\right\}\) be the output of Algorithm~\ref{algo:rounding} using 
 the notion of volume, boundary and expansion as in Definition~\ref{def:f-volume}, on a
 feasible solution to \ref{lp:f-ultrametric} and any choice of \(\varepsilon \in (0, 1)\). For 
 any \(t \in [m_\varepsilon]\), we have that \(x^t_{ij}\) is feasible
 for the layer-\(\lfloor (1 + \varepsilon)t \rfloor\) problem
 \ref{ilp:f-layer} and there is a constant \(c(\varepsilon) > 0\) depending only on \(\varepsilon\) such that
 \begin{align*}
  \sum_{\{i, j\} \in E(K_n)} \kappa(i, j) \left(f(t) - f(t-1)\right) x^t_{ij} \le \left(c(\varepsilon)\log{n}\right)
  \gamma_t.
 \end{align*}
\end{corollary}

Corollary~\ref{cor:f-approx} allows us to prove the analogue of Theorem~\ref{thm:main}, i.e., we can use 
Algorithm~\ref{algo:rounding} to get an ultrametric that is an \(f\)-image of a non-trivial ultrametric
and whose cost is at most \(O(\log{n})\) times the cost of an optimal hierarchical clustering according to 
cost function~\eqref{fcost}.

\begin{theorem}\label{thm:f-main}
 Let \(\{x^t_{ij} \mid t \in \left[m_\varepsilon\right], i,j \in V\}\) be the output of Algorithm~\ref{algo:rounding}
 using 
 the notion of volume, boundary, and expansion as in Definition~\ref{def:f-volume}
on an optimal solution \(\{d_t\}_{t\in [n-1]}\) of \ref{lp:f-ultrametric} for any choice of 
\(\varepsilon \in (0, 1)\).
Define the sequence \(\left\{y^t_{ij}\right\}\) for every \(t \in [n-1]\) and \(i, j \in V\) as
\begin{align*}
y^t_{ij} \coloneqq \begin{cases} 
		    x^{\lfloor t/(1 + \varepsilon) \rfloor}_{ij} \quad &\text{if  } t > 1 + \varepsilon\\
		    1 \quad &\text{if } t \le 1 + \varepsilon.
                   \end{cases}
\end{align*}
Then \(y^t_{ij}\) is feasible for \ref{ilp:f-ultrametric} and there is a constant \(c(\varepsilon) > 0\) such that
\begin{align*}
\sum_{t=1}^{n-1} \sum_{\{i, j\}\in E(K_n)} \kappa(i, j) \left(f(t) - f(t-1)\right)y^t_{ij} \le 
\left(c(\varepsilon) \log{n}\right) \OPT
\end{align*}
where \(\OPT\) is the optimal solution to \ref{ilp:f-ultrametric}.
\end{theorem}

\begin{proof}
 Immediate from Corollary~\ref{cor:f-approx} and Theorem~\ref{thm:main}.
\end{proof}

Finally we put everything together to obtain the corresponding Algorithm~\ref{algo:f-obtain-hierarchy}
that outputs a hierarchical clustering of \(V\) of cost at most \(O\left(\log{n}\right)\) times the optimal
clustering according to cost function~\eqref{fcost}.

\begin{corollary}
Given a data set \(V\) of \(n\) points and a similarity function \(\kappa: V \times V \to \R\),
Algorithm~\ref{algo:f-obtain-hierarchy} returns a hierarchical clustering \(T\) of \(V\) satisfying
\begin{align*}
 \cost_f(T) \le O\left(a_n +\log{n}\right) \min_{T' \in \mathcal{T}} \cost_f(T'),
\end{align*}
where \(a_n \coloneqq \max_{n' \in [n]} f(n') - f(n'-1)\).
Moreover Algorithm~\ref{algo:f-obtain-hierarchy} runs in time polynomial in 
\(n\), \(\log{f(n)}\) and \(\log\left(\max_{i, j \in V}\kappa(i, j)\right)\).
\end{corollary}

\begin{proof}
Let \(\widehat{T}\) be an optimal hierarchical clustering according to cost function~\eqref{fcost}.
By Corollary~\ref{cor:equivalent}, Lemma~\ref{lem:f-image-non-trivial} and Theorem~\ref{thm:f-main}
it follows that we can find a hierarchical clustering \(T\) satisfying
\begin{align*}
 \sum_{\{i, j\} \in E(K_n)} \kappa(i, j) f\left(\size{\leaves(T[\lca(i, j)]}-1\right)
 \le O(\log{n})\left(
 \sum_{\{i, j\} \in E(K_n)} \kappa(i, j) f\left(\size{\leaves(\widehat{T}[\lca(i, j)]}-1\right)\right).
\end{align*}
Recall that \(\cost_f(T) \coloneqq \sum_{\{i, j\} \in E(K_n)} \kappa(i, j) 
f\left(\size{\leaves(T[\lca(i, j)]}\right)\). Let \(K \coloneqq \sum_{\{i, j\} \in E(K_n)}
\kappa(i, j)\). Note that for any hierarchical clustering \(T'\) we have \(K \le \cost_f(T')\)
since \(f\) is an increasing function. From the above expression we infer that
\begin{align*}
 \cost_f(T) - a_n K \le \sum_{\{i, j\} \in E(K_n)} \kappa(i, j) f\left(\size{\leaves(T[\lca(i, j)]}-1\right)
 \le O(\log{n})\cost_f(\widehat{T}),
\end{align*}
and so \(\cost_f(T) \le O(\log{n}) \cost_f(\widehat{T}) + a_n K \le O(a_n + \log{n})\cost_f(\widehat{T})\).
We can find an optimal solution to \ref{lp:f-ultrametric} due to Lemma~\ref{lem:polytime} using the 
Ellipsoid algorithm in time polynomial in \(n\), \(\log{f(n)}\), and \(\log\left(\max_{i, j \in V}\kappa(i, j)\right)\). 
Note the additional \(\log{f(n)}\) in the running time since now 
we need to binary search over the interval \(\left[0, \max_{i, j \in V} \kappa(i, j)\cdot f(n) \cdot n\right]\).
Algorithm~\ref{algo:rounding} runs in time polynomial in \(n\) due to Theorem~\ref{thm:approx}.
Finally, Algorithm~\ref{algo:buildtree} runs in time \(O\left(n^3\right)\) due to Lemma~\ref{lem:bijection-ultrametric}.
\end{proof}

\begin{algorithm}[!htb]\label{algo:f-obtain-hierarchy}
\DontPrintSemicolon 
\KwIn{Data set \(V\) of \(n\) points, similarity function \(\kappa: V \times V \to \R_{\ge 0}\), \(f: \R_{\ge 0} \to 
\R_{\ge 0}\) strictly increasing with \(f(0) = 0\)}
\KwOut{Hierarchical clustering of \(V\)}
Solve \ref{lp:f-ultrametric} to obtain optimal sequence of spreading metrics \(\left\{d_t \mid d_t: V \times V \to 
[0, 1]\right\}\)\\
Fix a choice of \(\varepsilon \in (0, 1)\)\\
\(m_\varepsilon \gets \left\lfloor \frac{n-1}{1 + \varepsilon}\right\rfloor\)\\
Let \(\left\{x^t_{ij}\mid t \in [m_\varepsilon]\right\}\) be the output of 
Algorithm~\ref{algo:rounding} on \(V, \kappa, \{d_t\}_{t \in [n-1]}\)\\
Let \(y^t_{ij} \coloneqq \begin{cases} x^{\lfloor t/(1 + \varepsilon)\rfloor}_{ij} &\text{ if } t > 1 + \varepsilon
\\ 1 \qquad&\text{ if } t \le 1 + \varepsilon\end{cases}\)  for every \(t \in [n-1], i, j \in E(K_n)\)\\
\(d(i, j) \gets \sum_{t = 1}^{n-1} \left(f(t) - f(t-1)\right)y^t_{ij}\) for every \(i, j \in E(K_n)\)\\
\(d(i, i) \gets 0\) for every \(i \in V\)\\
Let \(r, T\) be the output of Algorithm~\ref{algo:buildtree} on \(V, f^{-1}(d)\)\\
\Return{\(r, T\)}\;
\caption{Hierarchical clustering of \(V\) for cost function~\eqref{fcost}}
\end{algorithm}

\section{Experiments}
Finally, we describe the experiments we performed.
For small data sets \ref{ilp:ultrametric} and \ref{ilp:f-ultrametric}
describe integer programming formulations that allow us to compute the exact optimal hierarchical
clustering for cost functions~\eqref{cost} and \eqref{fcost} respectively.
We implement \ref{ilp:f-ultrametric} where one can plug in any strictly increasing
function \(f\) satisfying \(f(0) = 0\). In particular, setting \(f(x) = x\) gives
us \ref{ilp:ultrametric}. We use the Mixed Integer Programming (MIP) solver Gurobi \(6.5\)
\cite{gurobi}. 
Similarly, we also implement Algorithms~\ref{algo:buildtree}, \ref{algo:rounding}, and 
\ref{algo:f-obtain-hierarchy} using Gurobi as our LP solver.
Note that Algorithm~\ref{algo:f-obtain-hierarchy} needs to fix a parameter choice
\(\varepsilon \in (0, 1)\). In Sections~\ref{sec:lp-rounding} and \ref{sec:f-lp-rounding} we did not discuss
the effect of the choice of the parameter \(\varepsilon\) in detail. In particular, 
we need to choose an \(\varepsilon\) small enough such that for every \(U \subseteq V\) encountered in 
Algorithm~\ref{algo:rounding}, \(\vol{\ball{i}{\Delta}{t}{U}}\) 
is of the same sign as \(\vol{\ball{i}{0}{t}{U}}\) for every \(t \in [n-1]\), so that 
\(\log\left(\frac{\vol{\ball{i}{\Delta}{t}{U}}}{\vol{\ball{i}{0}{t}{U}}}\right)\) is defined.
In our experiments we start with a particular value of \(\varepsilon\) (say \(0.5\))
and halve it till the volumes have the same sign. 
For the sake of exposition, we limit ourselves to the following choices for the function \(f\)
\begin{align*}
\left\{x, x^2, \log(1 + x), e^x -1\right\}.
\end{align*}
By Lemma~\ref{lem:polytime} we can optimize over \ref{lp:f-ultrametric} in time polynomial
in \(n\) using the Ellipsoid method. In practice however, we use the \emph{dual simplex} method
where we separate triangle inequality constraints~\eqref{lp:f-triangle} and spreading
constraints~\eqref{lp:f-spreading} to obtain fast computations.
For the similarity function \(\kappa: V \times V \to \R\) we
limit ourselves to using \emph{cosine similarity} and the \emph{Gaussian kernel} 
with \(\sigma = 1\). They are defined formally below.

\begin{definition}[Cosine similarity]
 Given a data set \(V \in \R^m\) for some \(m \ge 0\), the cosine similarity \(\kappa_{cos}\) is defined
 as \(\kappa_{cos}(x, y) \coloneqq \frac{\sprod{x}{y}}{\norm{x}\norm{y}}\).
\end{definition}

Since the LP rounding Algorithm~\ref{algo:rounding} assumes that \(\kappa \ge 0\) in practice
we implement \(1 + \kappa_{cos}\) rather than \(\kappa_{cos}\). 

\begin{definition}[Gaussian kernel]
 Given a data set \(V \in \R^m\) for some \(m \ge 0\), the Gaussian kernel \(\kappa_{gauss}\) with 
 standard deviation \(\sigma\) is defined as 
 \(\kappa_{gauss}(x, y) \coloneqq \exp\left(-\frac{\norm{x-y}^2}{2\sigma^2}\right)\).
\end{definition}
The main aim of our experiments was to answer the following two questions.
\begin{enumerate}
 \item How good is the hierarchal clustering obtained from Algorithm~\ref{algo:f-obtain-hierarchy}
 as opposed to the true optimal output by \ref{ilp:f-ultrametric}?
 
 \item How good does Algorithm~\ref{algo:f-obtain-hierarchy} perform compared to other hierarchical
 clustering methods? 
\end{enumerate}
For the first question, we are restricted to working with small data sets since computing an optimum 
solution to \ref{ilp:f-ultrametric} is expensive. In this case we consider synthetic data sets
of small size and samples of some data sets from the UCI database \cite{Lichman:2013}. The synthetic
data sets we consider are mixtures of Gaussians in various small dimensional spaces. 
Figure~\ref{fig:ipvslp} shows a comparison of the cost of the hierarchy 
(according to cost function~\eqref{fcost}) returned by solving \ref{ilp:f-ultrametric} and by 
Algorithm~\ref{algo:f-obtain-hierarchy} for
various forms of \(f\) when the similarity function is \(\kappa_{cos}\) and \(\kappa_{gauss}\).
Note that we normalize the cost of the tree returned by \ref{ilp:f-ultrametric} and 
Algorithm~\ref{algo:f-obtain-hierarchy} by the cost of the trivial clustering \(r, T^*\) 
where \(T^*\) is the star graph with \(V\) as its leaves and \(r\) as the internal node.
In other words \(d_{T^*}(i, j) = n - 1\) for every distinct pair 
\(i, j \in V\) and so the normalized cost of any tree lies in the interval \((0, 1]\).

For the study of the second question, we consider some of the popular algorithms for hierarchical clustering 
are \emph{single linkage}, \emph{average linkage}, \emph{complete linkage}, and \emph{Ward's method}
\cite{ward1963hierarchical}.
To get a numerical handle on how good a hierarchical clustering \(T\) of \(V\) is, we prune the tree
to get the \emph{best} \(k\) flat clusters and measure its error relative to the target clustering.
We use the following notion of error also known as \emph{Classification Error}
that is standard in the literature for hierarchical clustering (see, e.g., \cite{meilua2001experimental}). 
Note that we may think of a flat \(k\)-clustering 
of the data \(V\) as a function \(h\) mapping elements of \(V\) to a label set \(\mathcal{L} \coloneqq
\{1, \dots, k\}\). Let \(S_k\) denote the group of permutations on \(k\) letters.

\begin{definition}[Classification Error]
 Given a proposed clustering \(h: V \to \mathcal{L}\) its \emph{classification error} relative to a target clustering 
 \(g : V \to \mathcal{L}\) is denoted by \(\err{g, h}\) and is defined as 
 \begin{align*}
  \err{g, h} \coloneqq \min_{\sigma \in S_k} \left[ \Pr_{x \in V} [h(x) \neq \sigma(g(x))\right]. 
 \end{align*}
\end{definition}

\begin{figure}
\centering
\begin{minipage}{.5\textwidth}
  \centering
  \includegraphics[scale=0.44]{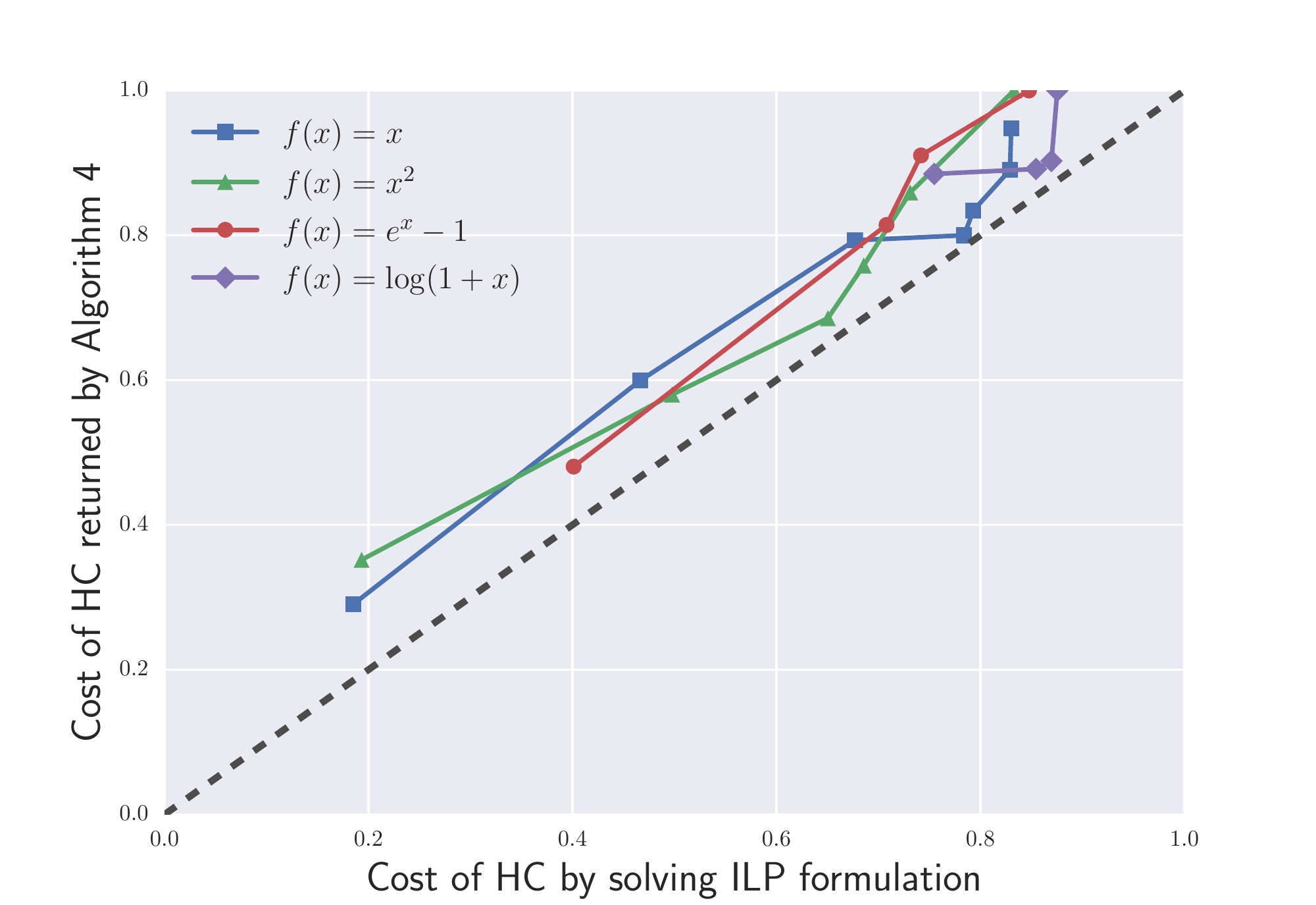}
\end{minipage}%
\begin{minipage}{.5\textwidth}
  \centering
  \includegraphics[scale=0.44]{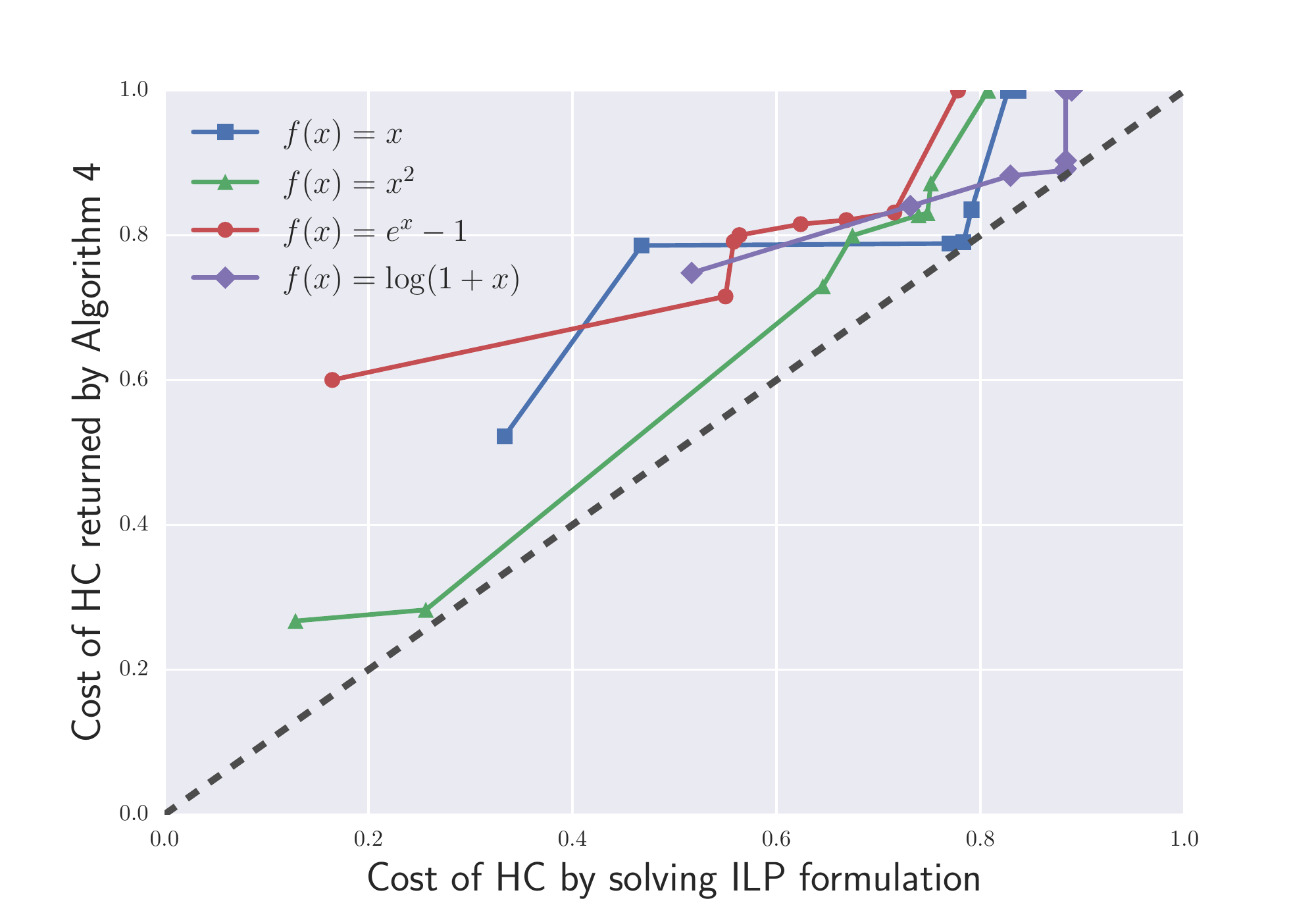}
  \end{minipage}
  \caption{Comparison of \ref{ilp:f-ultrametric} and Algorithm~\ref{algo:f-obtain-hierarchy} for 
  \(1 + \kappa_{cos}\) (left) and \(\kappa_{gauss}\) (right)}
  \label{fig:ipvslp}
\end{figure}
  
We compare the error of Algorithm~\ref{algo:f-obtain-hierarchy} with the 
various linkage based algorithms that are commonly used for 
hierarchical clustering, as well as Ward's method  and the \(k\)-means algorithm.
We test Algorithm~\ref{algo:f-obtain-hierarchy} most extensively for \(f(x) = x\) while
doing a smaller number of tests for \(f(x) \in \left\{x^2, \log(1 + x), e^x -1 \right\}\).
Note that both Ward's method and the \(k\)-means algorithm work on the squared Euclidean
distance \(\norm{x - y}_2^2\) between two points \(x, y \in V\), i.e., they both require an embedding of
the data points into a normed vector space which provides extra information that can be potentially
exploited. For the linkage based algorithms 
we use the same notion of similarity \(1 + \kappa_{cos}\) or \(\kappa_{gauss}\) that we use for Algorithm~\ref{algo:f-obtain-hierarchy}.
For comparison we use a mix of synthetic data sets as well as the Wine, Iris, Soybean-small,
Digits, Glass, and Wdbc data sets from the UCI repository \cite{Lichman:2013}. For some of the 
larger data sets, we sample uniformly at random a smaller number of data points and take the 
average of the error over the different runs. 
Figures~\ref{fig:error}, \ref{fig:error-quadratic}, \ref{fig:error-logarithm}, and
\ref{fig:error-exponential} show that the hierarchical clustering returned by Algorithm~\ref{algo:f-obtain-hierarchy} 
with \(f(x) \in \left\{x, x^2, \log(1 + x), e^x - 1\right\}\)
often has better projections into flat clusterings than the other algorithms. This is especially true
when we compare it to the linkage based algorithms, since they use the same pairwise similarity function as
Algorithm~\ref{algo:f-obtain-hierarchy}, as opposed to Ward's method and \(k\)-means.

\begin{figure}
\centering
\begin{minipage}{.5\textwidth}
  \centering
  \includegraphics[scale=0.44]{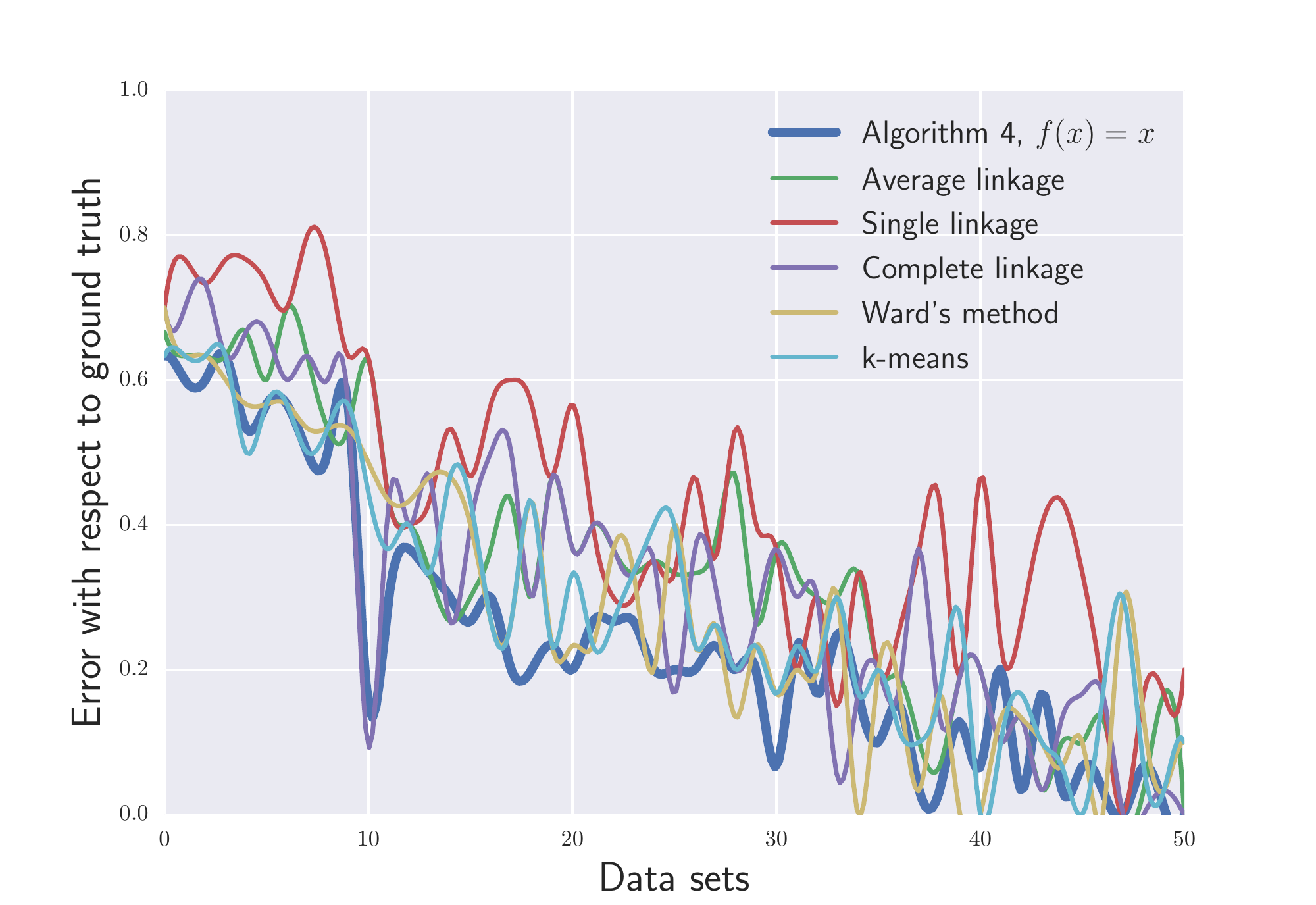}
\end{minipage}%
\begin{minipage}{.5\textwidth}
  \centering
  \includegraphics[scale=0.44]{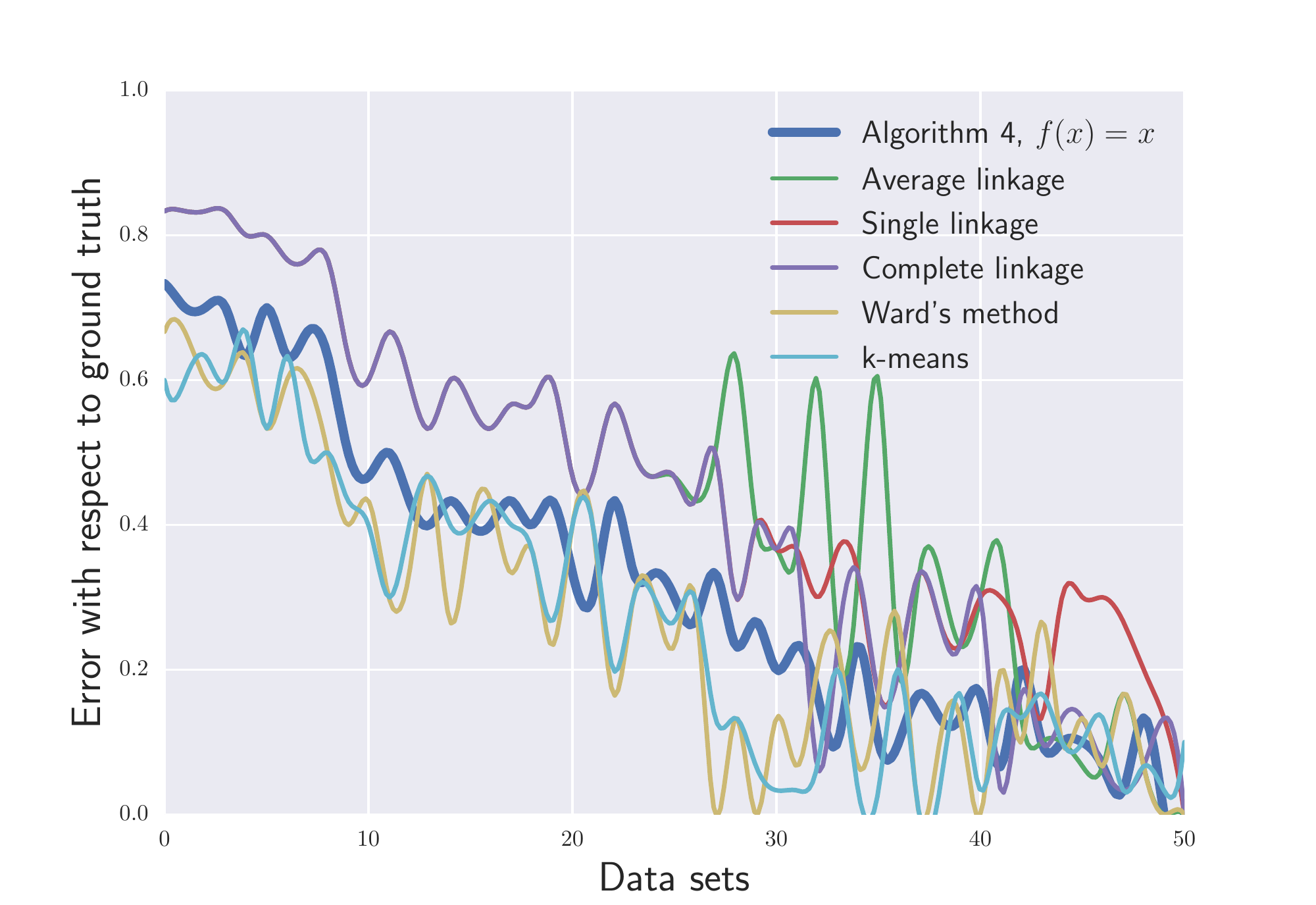}
  \end{minipage}
  \caption{Comparison of Algorithm~\ref{algo:f-obtain-hierarchy} using \(f(x) = x\), with other algorithms for clustering
  using \(1 + \kappa_{cos}\) (left) and \(\kappa_{gauss}\) (right)}
  \label{fig:error}
  \end{figure}

\begin{figure}
\centering
\begin{minipage}{.5\textwidth}
  \centering
  \includegraphics[scale=0.44]{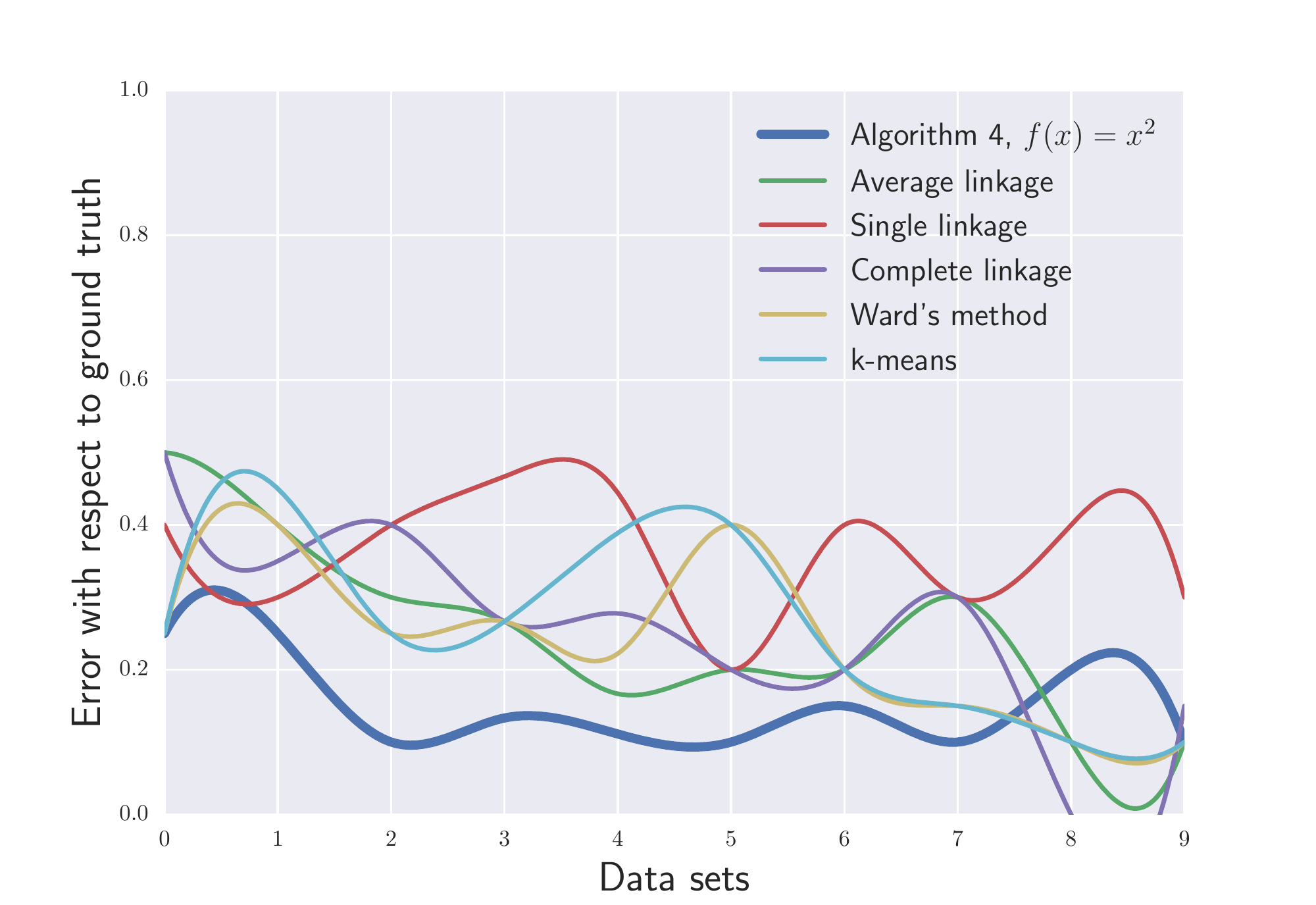}
\end{minipage}%
\begin{minipage}{.5\textwidth}
  \centering
  \includegraphics[scale=0.44]{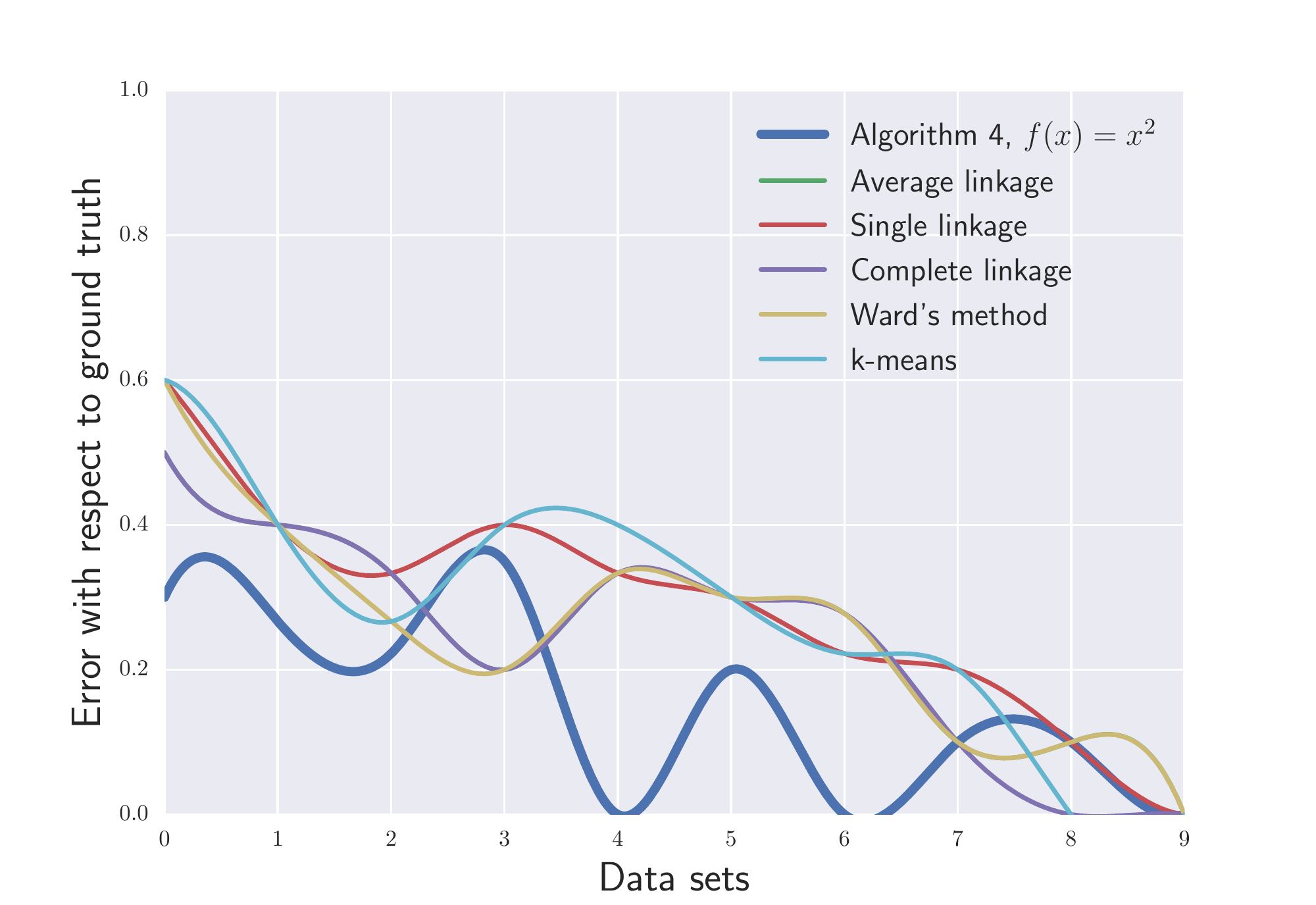}
  \end{minipage}
  \caption{Comparison of Algorithm~\ref{algo:f-obtain-hierarchy} using \(f(x) = x^2\), with other algorithms for clustering
  using \(1 + \kappa_{cos}\) (left) and \(\kappa_{gauss}\) (right)}
  \label{fig:error-quadratic}
  \end{figure}

 \begin{figure}
\centering
\begin{minipage}{.5\textwidth}
  \centering
  \includegraphics[scale=0.44]{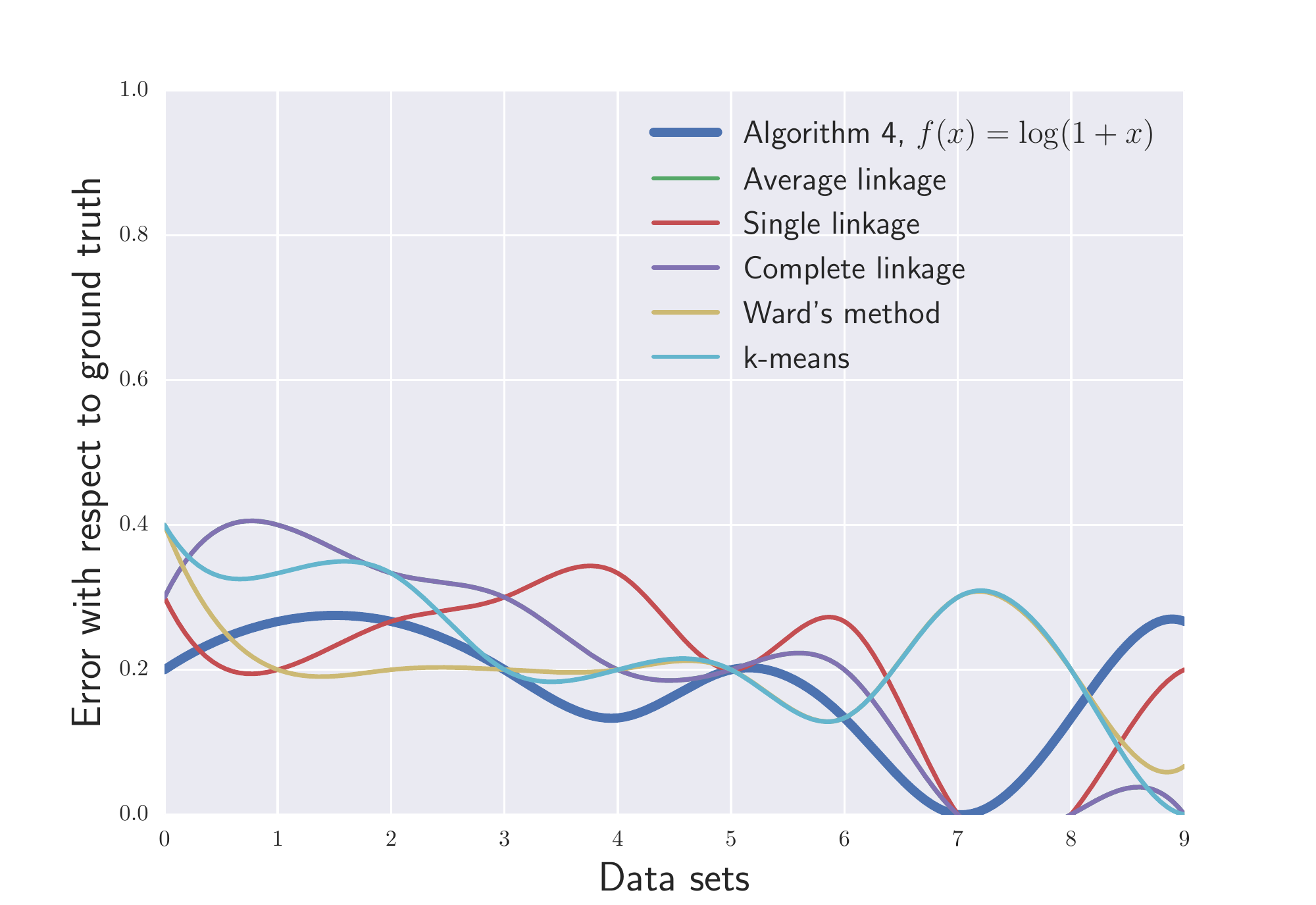}
\end{minipage}%
\begin{minipage}{.5\textwidth}
  \centering
  \includegraphics[scale=0.44]{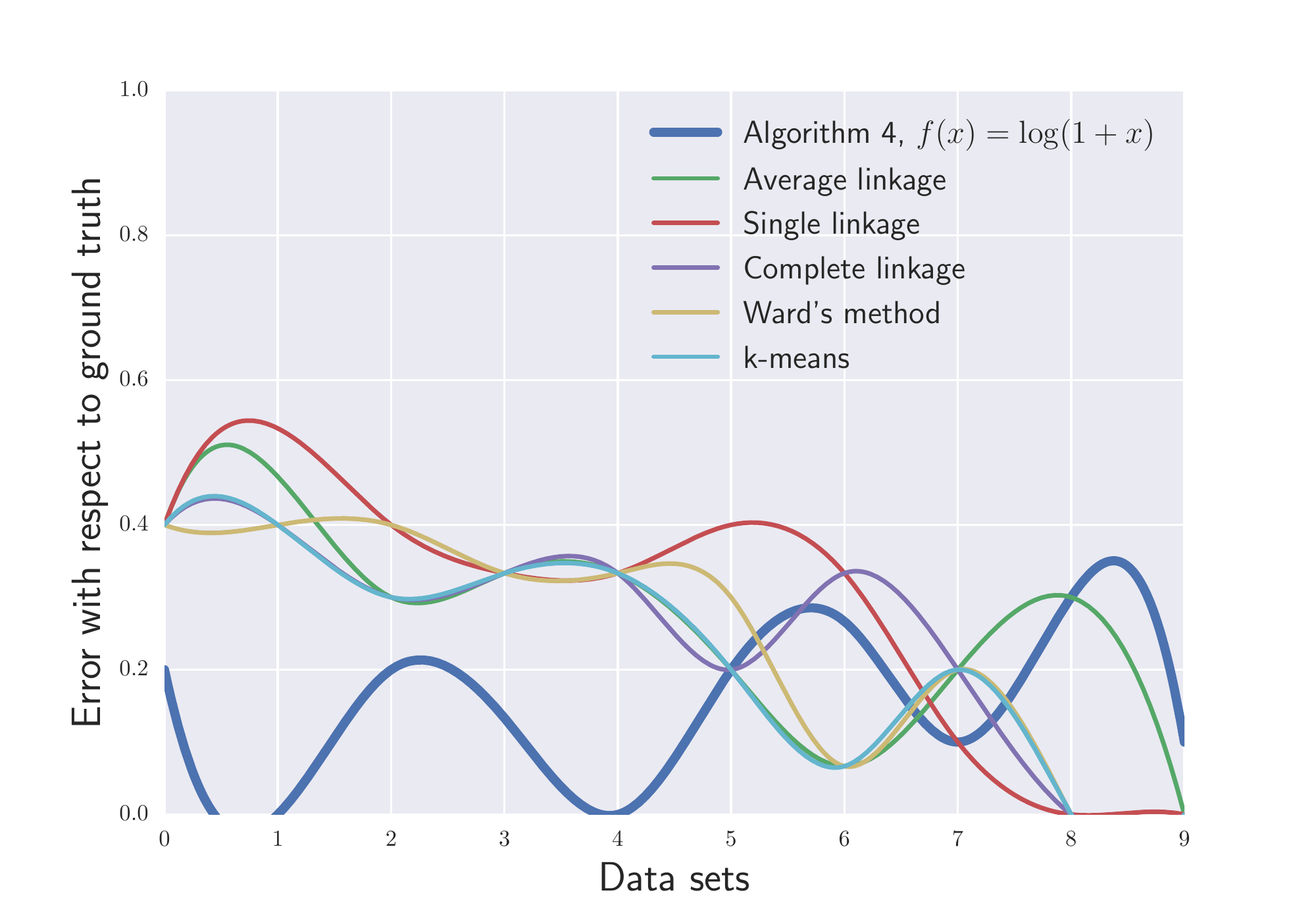}
  \end{minipage}
  \caption{Comparison of Algorithm~\ref{algo:f-obtain-hierarchy} using \(f(x) = \log(1 + x)\), with other algorithms for clustering
  using \(1 + \kappa_{cos}\) (left) and \(\kappa_{gauss}\) (right)}
  \label{fig:error-logarithm}
  \end{figure}

   \begin{figure}
\centering
\begin{minipage}{.5\textwidth}
  \centering
  \includegraphics[scale=0.44]{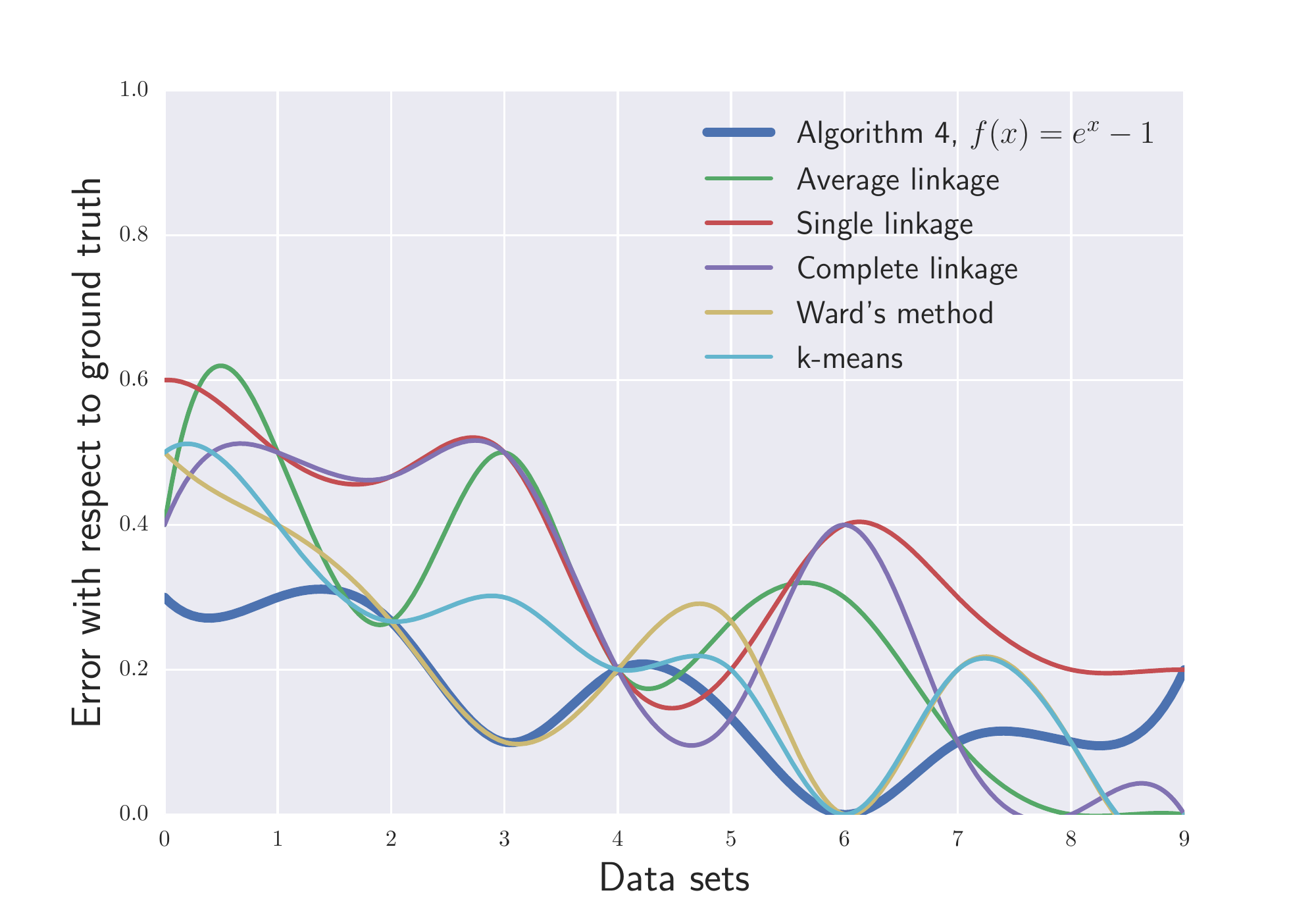}
\end{minipage}%
\begin{minipage}{.5\textwidth}
  \centering
  \includegraphics[scale=0.44]{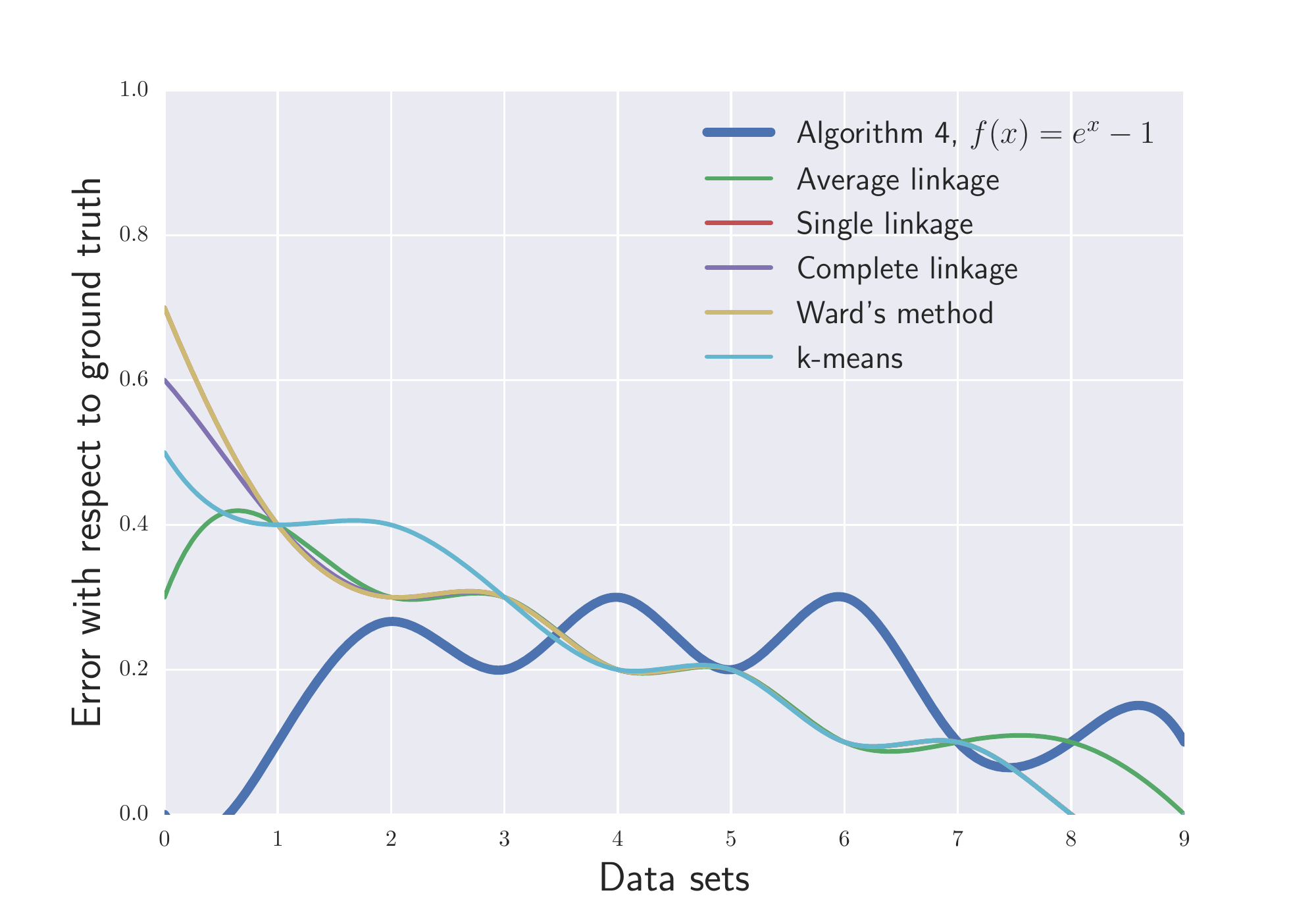}
  \end{minipage}
  \caption{Comparison of Algorithm~\ref{algo:f-obtain-hierarchy} using \(f(x) = e^x-1\), with other algorithms for clustering
  using \(1 + \kappa_{cos}\) (left) and \(\kappa_{gauss}\) (right)}
  \label{fig:error-exponential}
  \end{figure}

\section{Discussion}\label{sec:discussion}
In this work we have studied the cost functions~\eqref{cost} and \eqref{fcost} for hierarchical clustering 
given a pairwise similarity function over the data and shown an \(O(\log n)\) approximation algorithm for this problem.
As briefly mentioned in Section~\ref{sec:preliminaries} however, such a cost function is not unique.
Further, there is an intimate connection between hierarchical clusterings and ultrametrics over discrete sets
which points to other directions for formulating a cost function over hierarchies. In particular we 
briefly mention the related notion of \emph{hierarchically well-separated trees} (HST) as defined in \cite{bartal1996probabilistic}
(see also \cite{bartal2001ramsey, bartal2003metric}). A \(k\)-HST for \(k \ge 1\) is a tree
\(T\) such that each vertex \(u \in T\) has a label \(\Delta(u) \ge 0\) such that \(\Delta(u) = 0\) 
if and only if \(u\) is a leaf of \(T\). Further, if \(u\) is a child of \(v\) in \(T\) then 
\(\Delta(u) \le \Delta(v)/k\). It is well known that any ultrametric \(d\) on a finite set \(V\) is equivalent
to a \(1\)-HST where \(V\) is the set of leaves of \(T\) and \(d(i, j) = \Delta\left(\lca(i, j)\right)\)
for every \(i, j \in V\).
Thus in the special case when \(\Delta(u) = \size{\leaves{T[u]}} - 1\) we get the cost function~\eqref{cost},
while if \(\Delta(u) = f\left(\size{\leaves{T[u]}} - 1\right)\) for a strictly increasing function \(f\)
with \(f(0) = 0\) then we get cost function \eqref{fcost}. It turns out this assumption on \(\Delta\)
enables us to prove the combinatorial results of Section~\ref{sec:main} and give a \(O(\log n)\) approximation
algorithm to find the optimal cost tree according to these cost functions. 
It is an interesting problem to investigate cost functions and algorithms for hierarchical clustering induced
by other families of \(\Delta\) that arise from a \(k\)-HST on \(V\), i.e., if the cost of \(T\) is 
defined as 

\begin{align}
 \cost_{\Delta}(T) \coloneqq \sum_{\{i, j\} \in E(K_n)} \kappa(i, j) \Delta\left(\lca(i, j)\right).
\end{align}
Note that not all choices of \(\Delta\) lead to a meaningful cost function. For example, choosing 
\(\Delta(u) = \diam{T[u]}-1\) gives rise to the following cost function
\begin{align}
 \cost(T) \coloneqq \sum_{\{i, j\} \in E(K_n)} \kappa(i, j) \dist_T(i, j)\label{altcost}
\end{align}
where \(\dist_T(i, j)\) is the length of the unique path from \(i\) to \(j\) in \(T\). In this
case, the trivial clustering \(r, T^*\) where \(T^*\) is the star graph with \(V\) as its leaves
and \(r\) as the root is always a minimizer; in other words, there is no incentive for spreading 
out the hierarchical clustering. Also worth mentioning is a long line of related work on fitting
tree metrics to metric spaces (see e.g., \cite{ailon2005fitting, racke2008optimal, fakcharoenphol2003tight}).
In this setting, the data points \(V\) are assumed to come from a metric space \(d_V\) and the 
objective is to find a hierarchical clustering \(T\) so as to minimize \(\norm{d_V - d_T}_p\).
If the points in \(V\) lie on the unit sphere and the similarity function \(\kappa\)
is the cosine similarity \(\kappa_{cos}(i, j) =  1 - d_V(i, j)/2\), then the problem of fitting a tree metric
with \(p = 2\) minimizes the same objective as cost function~\eqref{altcost}. Since \(d_V \le 1\) in this case,
the minimizer is the trivial tree \(r, T^*\) (as remarked above). In general, when
the points in \(V\) are not constrained to lie on the unit sphere, the two problems are
incomparable.

\section{Acknowledgments}
Research reported in this paper was partially supported by NSF CAREER award CMMI-1452463 and NSF grant CMMI-1333789.
We would like to thank Kunal Talwar and Mohit Singh for helpful discussions and anonymous reviewers for helping improve the
presentation of this paper.

\bibliographystyle{apalike}
\bibliography{literature.bib}
\appendix
\section{Hardness of finding the optimal hierarchical clustering}
In this section we study the hardness of finding the optimal hierarchical 
clustering according to cost function~\eqref{cost}. 
We show that under the assumption of the \emph{Small Set Expansion} (\Problem{SSE}) hypothesis
there is no constant factor approximation algorithm for this problem. We
also show that no polynomial sized Linear Program (LP) or Semidefinite Program (SDP)
can give a constant factor approximation for this problem without the need for any 
complexity theoretic assumptions. Both these results
make use of the similarity of this problem with the \emph{minimum linear
arrangement} problem. To show hardness under Small Set Expansion, we make use of the result
of \cite{raghavendra2012reductions} showing that there is no constant factor approximation
algorithm for the Minimum Linear Arrangement problem under the assumption of \Problem{SSE}.
To show the LP and SDP inapproximability results, we make use of the reduction framework
of \cite{DBLP:journals/corr/BraunPR15} together with the NP-hardness proof for
Minimum Linear Arrangement due to \cite{garey1976some}. We also note that both these
hardness results hold even for unweighted graphs (i.e., when \(\kappa \in \{0, 1\}\)).

Note that the individual layer-\(t\) problem~\ref{ilp:f-layer} for \(t = \lfloor n/2\rfloor\)
is equivalent to the \emph{minimum bisection problem} for which the best known 
approximation is \(O(\log{n})\) due to \cite{racke2008optimal}, while 
the best known bi-criteria approximation is \(O\left(\sqrt{\log n}\right)\) due
to \cite{arora2009expander} and improving these approximation factors is a major open problem.
However it is not clear if an improved approximation algorithm for hierarchical clustering
under cost function~\eqref{cost}
would imply an improved algorithm for every layer-\(t\) problem, which is why a constant
factor inapproximability result is of interest. 
We start by recalling the definition of an \emph{optimization problem}
in the framework of \cite{DBLP:journals/corr/BraunPR15}.

\begin{definition}[Optimization problem]\cite{DBLP:journals/corr/BraunPR15}
  \label{def:opt-problem}
  An \emph{optimization problem}
  is a tuple
  \(\mathcal{P} = (\mathcal{S}, \mathfrak{I}, \val)\)
  consisting of a set \(\mathcal{S}\)
  of \emph{feasible solutions}, a set \(\mathfrak{I}\)
  of \emph{instances},
  and a real-valued objective
  called \emph{measure}
  \(\val \colon \mathfrak{I} \times \mathcal{S}  \to \R\).
We shall use \(\val_{\mathcal{I}}(s)\) for the objective value
of a feasible solution \(s \in \mathcal{S}\)
for an instance \(\mathcal{I} \in \mathfrak{I}\).
\end{definition}

Since we are interested in the integrality gaps of LP and SDP relaxations
for an optimization problem \(\mathcal{P}=(\mathcal{S}, \mathfrak{I}, \val)\), 
we represent the approximation gap by two functions \(C, S: \mathfrak{I} \to \R\)
where \(C\) is the \emph{completeness guarantee} while \(S\) is the \emph{soundness
guarantee}. Note that the ratio \(C/S\) represents the approximation factor for the
problem \(\mathcal{P}\).  We recall below the formal definition of an LP
relaxation of \(\mathcal{P}\) that achieves a \((C, S)\)-approximation guarantee.
We assume without loss of generality that \(\mathcal{P}\) is a maximization problem.

\begin{definition}[LP formulation of an optimization problem]\cite{DBLP:journals/corr/BraunPR15}
  \label{def:LP-formulation}
  Let
  \(\mathcal{P} = (\mathcal{S}, \mathfrak{I}, \val)\)
  be an optimization problem,
  and \(C, S: \mathfrak{I} \to \R\).
  Then let \(\mathfrak{I}^{S} \coloneqq
  \set{\mathcal{I} \in \mathfrak{I}}{\max
  \val_{\mathcal{I}} \leq S(\mathcal{I})}\)
  denote the set of \emph{sound} instances, i.e., 
  for which the soundness guarantee \(S\)
  is an upper bound on the maximum.
  A \emph{\((C, S)\)-approximate LP formulation} of \(\mathcal{P}\)
  consists of a linear program \(A x \leq b\)
  with \(x \in \R^{r}\) for some \(r\)
  and the following \emph{realizations}:
  \begin{description}
  \item[Feasible solutions] as vectors \(x^{s} \in \R^{r}\)
    for every \(s \in \mathcal{S}\) satisfying
  \begin{align}
    \label{eq:LP-contain}
    A x^{s} &\leq b \qquad \text{for all } s \in \mathcal{S},
  \end{align}
  i.e., the system \(Ax \leq b\) is a relaxation of
  \(\conv{x^s \mid s \in \mathcal{S}}\).
  \item[Instances] as affine functions
    \(w_{\mathcal{I}} \colon \R^{r} \to \R\)
    for all \(\mathcal{I} \in \mathfrak{I}^{S}\)
    satisfying
    \begin{align}
      \label{eq:LP-linear}
      w_{\mathcal{I}}(x^{s}) & = \val_{\mathcal{I}}(s)
            \qquad \text{for all } s \in
      \mathcal{S},
    \end{align}
    i.e., the linearization \(w_{\mathcal{I}}\) of
    \(\val_{\mathcal{I}}\)
    is required to be exact on all \(x^s\)
    with \(s \in \mathcal{S}\).
  \item[Achieving \((C,S)\) approximation guarantee]
  by requiring
  \begin{align}
    \label{eq:LP-approx}
    \max \set{w_{\mathcal{I}}(x)}{A x \leq b} &\leq C(\mathcal{I})
    \qquad \text{for all } \mathcal{I} \in \mathfrak{I}^{S},
  \end{align}
  \end{description}
  The \emph{size} of the formulation is the number of inequalities
  in \(A x \leq b\).
  Finally, the
  \((C, S)\)-approximate
  \emph{LP formulation complexity} \(\fc(\mathcal{P}, C, S)\)
  of \(\mathcal{P}\) is
  the minimal size of all its LP formulations.
\end{definition}

One can similarly define a \((C, S)\)-approximate SDP formulation for a
problem \(\mathcal{P}\) where instead of a LP, we now have a
SDP relaxation \(\mathcal{A}(X) = b\) with \(X \in \symM^r_{+}\) and where
\(\symM^r_+\) denotes the space of \(r \times r\) positive semidefinite matrices.
The size of such an SDP formulation is measured by the dimension \(r\) and
\(\fcSDP(\mathcal{P}, C, S)\) is defined as the minimum size of an SDP formulation
achieving \((C, S)\)-approximation for problem \(\mathcal{P}\). Below we recall
the precise notion of a reduction between two problems as in \cite{DBLP:journals/corr/BraunPR15}.

\begin{definition}[Reduction]\cite{DBLP:journals/corr/BraunPR15}
  \label{def:red-simple}
  Let \(\mathcal{P}_{1} = (\mathcal{S}_{1}, \mathfrak{I}_{1}, \val)\)
  and
  \(\mathcal{P}_{2} = (\mathcal{S}_{2}, \mathfrak{I}_{2}, \val)\) be
  optimization problems with guarantees
  \(C_{1}, S_{1}\) and \(C_{2}, S_{2}\), respectively.
  Let \(\tau_{1} = +1\) if \(\mathcal{P}_{1}\) is a maximization
  problem, and \(\tau_{1} = -1\) if \(\mathcal{P}_{1}\) is a
  minimization problem.  Similarly, let \(\tau_{2} = \pm 1\)
  depending on whether \(\mathcal{P}_{2}\) is a maximization problem
  or a minimization problem.

  A \emph{reduction} from \(\mathcal{P}_{1}\)
  to \(\mathcal{P}_{2}\)
  respecting the guarantees
  consists of
  \begin{enumerate}
  \item
    two mappings:
    \(* \colon \mathfrak{I}_{1} \to \mathfrak{I}_{2}\)
    and
    \(* \colon \mathcal{S}_{1} \to \mathcal{S}_{2}\)
    translating instances and feasible solutions independently;
  \item
    two nonnegative \(\mathfrak{I}_{1} \times \mathcal{S}_{1}\)
    matrices \(M_{1}\), \(M_{2}\)
  \end{enumerate}
  subject to the conditions
  \begin{subequations}\label{eq:red-simple}
  \begin{align}
    \label{eq:red-simple-complete}
    \tau_{1}
    \left[
      C_{1}(\mathcal{I}_{1}) - \val_{\mathcal{I}_{1}}(s_{1})
    \right]
    &
    =
    \tau_{2}
    \left[
      C_{2}(\mathcal{I}_{1}^{*}) - \val_{\mathcal{I}_{1}^{*}}(s_{1}^{*})
    \right]
    M_{1}(\mathcal{I}_{1}, s_{1})
    +
    M_{2}(\mathcal{I}_{1}, s_{1})
    \tag{\theparentequation-complete}
    \\
    \label{eq:red-simple-sound}
    \tau_{2} \OPT\left(\mathcal{I}_{1}^{*}\right) &\leq \tau_{2}
    S_{2}(\mathcal{I}_{1}^{*})
    \qquad
    \text{if \(\tau_{1} \OPT\left(\mathcal{I}_{1}\right) \leq
    \tau_{1} S_{1}(\mathcal{I}_{1})\).}
    \tag{\theparentequation-sound}
  \end{align}
  \end{subequations}
\end{definition}
The matrices \(M_1\) and \(M_2\) control the parameters of the reduction relating
the integrality gap of relaxations for \(\mathcal{P}_1\) to the integrality gap of
corresponding relaxations for \(\mathcal{P}_2\).
For a matrix \(A\), let \(\nnegrk A\) and \(\psdrk A\) denote the nonnegative rank and psd rank 
of \(A\) respectively.
The following theorem is a restatement of Theorem 3.2 from \cite{DBLP:journals/corr/BraunPR15} ignoring
constants. 
\begin{theorem}\cite{DBLP:journals/corr/BraunPR15}
  \label{thm:red-simple}
  Let \(\mathcal{P}_{1}\) and \(\mathcal{P}_{2}\) be optimization
  problems with a reduction from \(\mathcal{P}_{1}\)
  to \(\mathcal{P}_{2}\) respecting the
  completeness guarantees \(C_{1}\), \(C_{2}\)
  and soundness guarantees \(S_{1}\), \(S_{2}\)
  of \(\mathcal{P}_{1}\) and \(\mathcal{P}_{2}\), respectively.
  Then
  \begin{align}
    \fc(\mathcal{P}_{1}, C_{1}, S_{1})
    &
    \leq
    \nnegrk M_{2} +
    \nnegrk M_{1}
    + \nnegrk M_{1} \cdot \fc(\mathcal{P}_{2}, C_{2}, S_{2}),
    \\
    \fcSDP(\mathcal{P}_{1}, C_{1}, S_{1})
    &
    \leq
    \psdrk M_{2} +
    \psdrk M_{1}
    + \psdrk M_{1} \cdot \fcSDP(\mathcal{P}_{2}, C_{2}, S_{2}),
  \end{align}
  where \(M_{1}\) and \(M_{2}\) are the matrices in the reduction
  as in Definition~\ref{def:red-simple}.
\end{theorem}

Therefore to obtain a lower bound for problem \(\mathcal{P}_2\), it 
suffices to find a source problem \(\mathcal{P}_1\) and matrices \(M_1\)
and \(M_2\) of low nonnegative rank and low psd rank, satisfying Definition~\ref{def:red-simple}.

Below, we cast the hierarchical clustering problem (\Problem{HCLUST}) as an optimization problem.
We also recall a different formulation of cost function~\eqref{cost} due to \cite{DBLP:conf/stoc/Dasgupta16}
that will be useful in the analysis of the reduction.

\begin{definition}[\(\Problem{HCLUST}\) as optimization problem]\label{def:hclust}
The minimization problem \(\Problem{HCLUST}\) of size \(n\) consists of 
 \begin{description}
 \item[instances] similarity function \(\kappa: E(K_n) \to \R_{\ge 0}\)
 
 \item[feasible solutions] hierarchical clustering \(r, T\) of \(V(K_n)\)
 
 \item[measure] \(\val_{\kappa}(T) = \sum_{\{i, j\} \in E(K_n)} \kappa(i, j) \size{\leaves(T[\lca(i, j)])}\).
\end{description}
\end{definition}

We will also make use of 
the following alternate interpretation of cost function~\eqref{cost} given by \cite{DBLP:conf/stoc/Dasgupta16}.
Let \(\kappa \colon V \times V \to \R_{\ge 0}\) be an instance of \Problem{HCLUST}. For a subset \(S \subseteq V\),
a split \(S_1, \dots, S_k\) is a partition of \(S\) into \(k\) disjoint pieces. For a binary split \(S_1, S_2\) we can
define \(\kappa(S_1, S_2) \coloneqq \sum_{i \in S_1, j \in S_2} \kappa(i, j)\). This can be extended
to \(k\)-way splits in the natural way:
\begin{align*}
\kappa(S_1, \dots, S_k) \coloneqq \sum_{1 \le i \le j \le k} \kappa(S_i, S_j).
\end{align*}
Then the cost of a tree \(T\) is the sum over all the internal nodes of the splitting costs
at the nodes, as follows.

\begin{align*}
 \cost(T) = \sum_{\text{splits } S \rightarrow (S_1, \dots, S_k) \text{ in } T} \size{S} \kappa(S_1, \dots, S_k).
\end{align*}

We now briefly recall the \Problem{MAXCUT} problem.

\begin{definition}[\Problem{MAXCUT} as optimization problem]
The maximization problem \Problem{MAXCUT} of size \(n\) consists of
\begin{description}
\item[instances] all graphs \(G\) with \(V(G) \subseteq [n]\)

\item[feasible solutions] all subsets \(X\) of \([n]\)

\item[measure] \(\val_{G}(X) = \size{\delta_G(X)}\).
\end{description}
\end{definition}

Similarly, the Minimum Linear Arrangement problem can be phrased as an optimization problem 
as follows.

\begin{definition}[\Problem{MLA} as optimization problem]\label{def:mla}
The minimization problem \Problem{MLA} of size \(n\) consists of
\begin{description}
\item[instances] weight function \(w : E(K_n) \to \R_{\ge 0}\)

\item[feasible solutions] all permutations \(\pi : V(K_n) \to [n]\)

\item[measure] \(\val_{w}(\pi) \coloneqq \sum_{\{i, j \} \in E(K_n)} w(i, j)
\size{\pi(i) - \pi(j)}\).
\end{description}
\end{definition}

We now describe the reduction from \Problem{MAXCUT} to \Problem{HCLUST} which is 
a modification of the reduction from \Problem{MAXCUT} to \Problem{MLA} due to \cite{garey1976some}.
Note that an instance of \Problem{MAXCUT} maps to an unweighted instance of \Problem{HCLUST}, i.e.,
\(\kappa \in \{0, 1\}\).

\begin{description}
\item[Mapping instances] Given an instance \(G = (V, E)\) of \Problem{MAXCUT} of size \(n\), let 
\(r = n^4\) and \(U = \{u_1, u_2, \dots, u_r\}\). The instance \(\kappa\) of \Problem{HCLUST} is on the graph 
with vertex set \(V' \coloneqq V \cup U\) and has weights in \(\{0, 1\}\). For any distinct
pair \(i, j \in V'\), if \(\{i, j\} \in E\) then we define \(\kappa(i, j) \coloneqq 0\) and
otherwise we set \(\kappa(i, j) \coloneqq 1\).

\item[Mapping solutions] Given a cut \(X \subseteq V\) of \Problem{MAXCUT} we map it to the 
clustering \(r, T\) of \(V'\) where the root \(r\) has the following children: \(n^4\) 
leaves corresponding to \(U\), and \(2\) internal vertices corresponding to \(X\) and \(\overline{X}\).
The internal vertices for \(X\) and \(\overline{X}\) are
split into \(\size{X}\) and \(\size{\overline{X}}\) leaves respectively at 
the next level.

\end{description}

The following lemma relates the LP and SDP formulations for \Problem{MAXCUT} and \Problem{MLA}.

\begin{lemma}\label{lem:maxcut-mla}
For any completeness and soundness guarantee \((C, S)\), we have the following
\begin{align*}
    \fc{(\Problem{MAXCUT}, C, S)} &\le \fc{\left(\Problem{HCLUST}, C', 
    S'\right)} + O(n^2)\\
    \fcSDP{(\Problem{MAXCUT}, C, S)} &\le \fcSDP{\left(\Problem{HCLUST}, C', 
    S'\right)} + O(n^2).
\end{align*}
where \(C' \coloneqq \frac{(n^4 + n)^3 - (n^4 + n)}{3} -C(n^4 +n)\) and \(S' \coloneqq 
\binom{n^4 + n + 1}{3} - Sn^4\).
\end{lemma}

\begin{proof}
To show completeness, we analyze the cost of the tree \(T\) that a cut \(X\) maps to, using the alternate
interpretation of the cost function~\eqref{cost} due to \cite{DBLP:conf/stoc/Dasgupta16} (see above). 
Let \(H\) be the graph on vertex set \(V'\) induced by \(\kappa\), i.e. \(\{i, j\} \in E(H)\)
iff \(\kappa(i, j) = 1\). Let \(\overline{H}\) denote the complement graph of \(H\) 
and let \(\overline{\kappa}\) be the similarity function induced by it, i.e., 
\(\overline{\kappa}(i, j) = 1\) iff \(\{i, j\} \not\in E(H)\) and 
\(\overline{\kappa}(i, j) = 0\) otherwise.
For a hierarchical clustering \(T\) of \(V'\), we denote by \(\cost_H(T)\) and \(\cost_{\overline{H}}(T)\)
the cost of \(T\) induced by \(\kappa\) and \(\overline{\kappa}\) respectively, i.e.,
\(\cost_{H}(T) \coloneqq \sum_{\{i, j\} \in E(H)} \size{\leaves(T[\lca(i, j)])}\)
and \(\cost_{\overline{H}}(T) \coloneqq \sum_{\{i, j\} \not\in E(H)}
\size{\leaves(T[\lca(i, j)])}\).
Let \(\overline{X} \coloneqq V' \setminus X\). The cost of the tree \(T\) that the cut \(X\) maps to,
is given by 
\begin{align*}
 \cost(T) &= \cost_H(T) \\ 
  &= \frac{\left(n + n^4\right)^3 - (n + n^4)}{3} - \cost_{\overline{H}}(T)\\
  &= \frac{\left(n + n^4\right)^3 - (n + n^4)}{3} - \sum_{\text{splits } 
  S \rightarrow (S_1, \dots, S_k) \text{ in } T} \size{S} 
 \overline{\kappa}(S_1, \dots, S_k)\\
 &= \frac{\left(n + n^4\right)^3 - (n + n^4)}{3} - \left(n + n^4\right)\val_G(X) - 
 \left(\size{X}\size{E[X]} + \size{\overline{X}}\size{E[\overline{X}]}\right),
 \end{align*}
where \(E[X]\) and \(E[\overline{X}]\) 
are the edges of \(E(H)\) induced on the set \(X\) and \(\overline{X}\) respectively.
Therefore, we have the following completeness relationship between the two problems

\begin{align*}
 C - \val_G(X) = \frac{1}{n + n^4} \left(\cost(T) - \left(\frac{(n + n^4)^3 - (n + n^4)}{3}-C(n+n^4)\right)\right)
  + \frac{\size{X}\size{E[X]} + \size{\overline{X}}\size{E[\overline{X}]}}{n^4 + n}.
\end{align*}

We now define the matrices \(M_1\) and \(M_2\) as \(M_1(H, X) \coloneqq \frac{1}{n + n^4}\) and
\(M_2(H, X) \coloneqq \size{X}\size{E[X]} + \size{\overline{X}}\size{E[\overline{X}]}\).
Clearly, \(M_1\) has \(O(1)\) nonnegative rank and psd rank. We claim that the nonnegative rank of 
\(M_2\) is at most \(2\binom{n}{2}\). The vectors \(v_H \in \R^{2\binom{n}{2}}\)
corresponding to the instances \(H\) is defined as the concatenation \([u_H, w_H]\) of two vectors
\(u_H, w_H \in \R^{\binom{n}{2}}\). Both the vectors \(u_H, w_H\) encode the edges of \(H\) scaled
by \(n^4 + n\), i.e.,
\(u_H(\{i, j\}) = w_H(\{i, j\}) = 1/(n^4 + n)\) iff \(\{i, j\} \in E(H)\) and \(0\) otherwise. 
The vectors \(v_{X} \in \R^{2\binom{n}{2}}\) corresponding to the solutions are also 
defined as the concatenation \([u_X, w_X]\) of two vectors \(u_X, w_X \in \R^n\).
The vector \(u_X\) encodes the vertices in \(X\) scaled by \(\size{X}\)
i.e., \(u_X(\{i, j\}) = \size{X}\) iff \(i, j \in X\) and \(0\)
otherwise. 
The vector \(w_X\) encodes the vertices in \(\overline{X}\) scaled by \(\size{\overline{X}}\)
i.e., \(w_X(\{i, j\}) = \size{\overline{X}}\) iff \(i, j \in \overline{X}\) and \(0\) otherwise.
Clearly, we have \(M_2(H, X) = \sprod{v_H}{v_X}\) and so the nonnegative (and psd) rank of
\(M_2\) is at most \(2\binom{n}{2}\).

Soundness follows due to the analysis in \cite{garey1976some} and by noting that the cost of 
a linear arrangement obtained by projecting the leaves of \(T\) is a lower bound on \(\cost(T)\).
By the analysis in \cite{garey1976some} if the optimal value \(\OPT(G)\) of \Problem{MAXCUT} 
is at most \(S\), then the optimal value of \Problem{MLA} on \(V', \kappa\) is at least 
\(\binom{n^4 + n + 1}{3} - Sn^4\). Therefore, it follows that the optimal value of \Problem{HCLUST}
on \(V', \kappa\) is also at least \(\binom{n^4 + n + 1}{3} - Sn^4\).
\end{proof}

The constant factor inapproximability result for \Problem{HCLUST} now follows due to the following
theorems.

\begin{theorem}[{\cite[Theorem~3.2]{chan2013approximate}}]
\label{thm:LP-MaxCUT}
For any \(\varepsilon > 0\) there are infinitely many \(n\) such
that
\begin{align*}
\fc\left(\Problem{MAXCUT}, 1 - \varepsilon, \frac{1}{2} +
\frac{\varepsilon}{6}\right) \ge n^{\Omega\left(\log{n}/\log\log{n}\right)}.
\end{align*}
\end{theorem}

\begin{theorem}[{\cite[Theorem~7.1]{DBLP:journals/corr/BraunPR15}}]
  \label{thm:SDP-MaxCUT}
  For any \(\delta, \varepsilon > 0\)
  there are infinitely many \(n\) such that
  \begin{equation}
    \label{eq:SDP-MaxCUT}
    \fcSDP\left(\Problem{MAXCUT}, \frac{4}{5} - \varepsilon, \frac{3}{4} + \delta\right)
    = n^{\Omega(\log n / \log \log n)}
    .
  \end{equation}
\end{theorem}

Thus we have the following corollary about the LP and SDP inapproximability for the problem
\Problem{HCLUST}.

\begin{corollary}[LP and SDP hardness for \Problem{HCLUST}]
\label{thm:hclust}
For any constant \(c \ge 1\), \Problem{HCLUST} is LP-hard and SDP-hard with an inapproximability
factor of \(c\).
\end{corollary}

\begin{proof}
Straightforward by using Theorems~\ref{thm:LP-MaxCUT} and \ref{thm:SDP-MaxCUT} together with Lemma~\ref{lem:maxcut-mla}
and by choosing \(n\) large enough.
\end{proof}

The following lemma shows that a minor modification of the
argument in \cite{raghavendra2012reductions} also implies a constant 
factor inapproximability result under the \emph{Small Set Expansion} (\Problem{SSE})
hypothesis. Note that this reduction is also true for unit capacity graphs, i.e.,
\(\kappa \in \{0, 1\}\). We briefly recall the formulation of the Small Set Expansion
hypothesis. Informally, given a graph \(G = (V, E)\) the problem is to decide whether
all ``small'' sets in the graph are expanding. 
Let \(d(i)\) denote the degree of a vertex \(i \in V\). 
For a subset \(S \subseteq V\) let \(\mu(S) \coloneqq \size{S}/\size{V}\) 
be the volume of \(S\), and let \(\phi(S) \coloneqq E(S, \overline{S})/\sum_{i \in S} d(i)\)
be the expansion of \(S\). Then the \Problem{SSE} problem is defined as follows.

\begin{definition}[Small set expansion (\Problem{SSE}) hypothesis \cite{raghavendra2012reductions}]
 For every constant \(\eta > 0\), there exists sufficiently small \(\delta > 0\)
 such that given a graph \(G = (V, E)\), it is NP-hard to decide the following cases,
 \begin{description}
  \item[Completeness] there exists a subset \(S \subseteq V\) with volume \(\mu(S) = \delta\)
  and expansion \(\phi(S) \le \eta\),
  
  \item[Soundness] every subset \(S \subseteq V\) of volume \(\mu(S) = \delta\) has 
  expansion \(\phi(S) \ge 1 - \eta\).
 \end{description}
\end{definition}

Under this assumption, \cite{raghavendra2012reductions} proved the following amplification
result about the expansion of small sets in the graph. 

\begin{theorem}[Theorem 3.5~\cite{raghavendra2012reductions}]\label{thm:sse}
 For all \(q \in \mathbb{N}\) and \(\varepsilon', \gamma > 0\) it is \Problem{SSE}-hard to distinguish
 the following for a given graph \(H = (V_H, E_H)\)
 \begin{description}
  \item[Completeness] There exist disjoint sets \(S_1, \dots, S_q \subseteq V_H\) satisfying
  \(\mu(S_i) = \frac{1}{q}\) and \(\phi(S_i) \le \varepsilon' + o(\varepsilon')\) for all \(i \in [n]\),
  
  \item[Soundness] For all sets \(S \subseteq V_H\) we have \(\phi(S) \ge 
  \phi_\mathcal{G}(1-\varepsilon'/2)(\mu(S)) - \gamma/\mu(S)\),
 \end{description}
 where \(\phi_\mathcal{G}(1-\varepsilon'/2)(\mu(S))\) is the expansion of sets of volume
 \(\mu(S)\) in the infinite Gaussian graph \(\mathcal{G}(1-\varepsilon'/2)\).
\end{theorem}

The following lemma establishes that it is \Problem{SSE}-hard to approximate \Problem{HCLUST} to
within any constant factor. The argument closely parallels Corollary~A.5 of \cite{raghavendra2012reductions}
where it was shown that it is \Problem{SSE}-hard to approximate \Problem{MLA} to within any constant factor.

\begin{lemma}\label{lem:l-reduction}
Let \(G = (V, E)\) be a graph on \(V\) with \(\kappa\) induced by the edges \(E\) i.e., \(\kappa(i, j) = 1\) iff 
\(\{i, j\} \in E\) and \(0\) otherwise. Then it is 
\Problem{SSE}-hard to distinguish between the following two cases
\begin{description}
 \item[Completeness] There exists a hierarchical clustering \(T\) of \(V\) with \(\cost(T) \le \varepsilon n\size{E}\),
 
 \item[Soundness] Every hierarchical clustering \(T\) of \(V\) satisfies \(\cost(T) \ge c\sqrt{\varepsilon} n\size{E}\)
\end{description}
for some constant \(c\) not depending on \(n\).
\end{lemma}

\begin{proof}
Apply Theorem~\ref{thm:sse} on the graph \(G\) with the following choice of parameters: \(q = \lceil 2/\varepsilon\rceil\),
\(\varepsilon' = \varepsilon/3\) and \(\gamma = \varepsilon\). Suppose there exist \(S_1, \dots, S_q \subseteq V\) 
satisfying \(\phi(S_i) \le \varepsilon' + o(\varepsilon')\) and \(\size{S_i} = \size{V}/q \le \varepsilon\size{V}/2\). 
Then consider the tree \(r, T\) with the root \(r\) having \(q\) children
corresponding to each \(S_i\), and each \(S_i\) being further separated into \(\size{S_i}\) leaves
at the next level. We claim that \(\cost(T) \le \varepsilon n \size{E}\). We analyze this using the
alternate interpretation of cost function~\eqref{cost} (see above). Every crossing edge between \(S_i, S_j\) 
for distinct \(i, j \in [q]\) incurs a cost of \(n\), but by assumption there are at most 
\(\varepsilon\size{E}/2\) such edges. Further, any edge in \(S_i\) incurs a cost \(\frac{n}{q} \le 
\varepsilon n/2\) and thus their contribution is upper bounded by \(\varepsilon n\size{E}\). 

The analysis for soundness follows by the argument of Corollary~A.5 in \cite{raghavendra2012reductions}.
In particular, if for every \(S \subseteq V\) we have \(\phi(S) \ge 
  \phi_\mathcal{G}(1-\varepsilon'/2)(\mu(S)) - \gamma/\mu(S)\) then the cost of the optimal
  linear arrangement on \(G\) is at most \(\sqrt{\varepsilon} n \size{E}\). Since the cost of any tree
  (including the optimal tree) is at least the cost of the linear arrangement induced by projecting the leaf
  vertices, the claim about soundness follows.
\end{proof}

 \end{document}